\documentclass[pdflatex,sn-mathphys-ay]{sn-jnl}%

\usepackage{graphicx}%
\usepackage{multirow}%
\usepackage{amsmath,amssymb,amsfonts}%
\usepackage{amsthm}%
\usepackage{mathrsfs}%
\usepackage[title]{appendix}%
\usepackage{xcolor}%
\usepackage{textcomp}%
\usepackage{manyfoot}%
\usepackage{booktabs}%
\usepackage{algorithm}%
\usepackage{algorithmicx}%
\usepackage{algpseudocode}%
\usepackage{listings}%

\usepackage{afterpage}
\usepackage{amssymb}
\usepackage{amsmath}
\usepackage{array}
\usepackage{booktabs}
\usepackage{floatpag}
\usepackage{pgfplots}
\usepackage{pdflscape}
\usepackage{tikz}
\usepackage{siunitx}

\theoremstyle{thmstyleone}%
\newtheorem{theorem}{Theorem}%  meant for continuous numbers

\theoremstyle{thmstyletwo}%

\theoremstyle{thmstylethree}%
\newtheorem{lemma}[theorem]{Lemma}

\def\ba{\begin{array}}
\def\ea{\end{array}}
\def\beq{\begin{equation}}
\def\eeq{\end{equation}}
\def\beqnn{\begin{equation*}}
\def\eeqnn{\end{equation*}}
\def\beann{\begin{eqnarray*} }
\def\eeann{\end{eqnarray*}}
\def\bea{\begin{eqnarray}}
\def\eea{\end{eqnarray}}
\def\bmi{\begin{minipage}[t]}
\def\emi{\end{minipage}}

\newcommand{\SE}{\mathrm{SE}}
\newcommand{\MSE}{\mathrm{MSE}}
\newcommand{\errG}{\mbox{infeas}_{G}}

\newcommand{\D}{Z}

\newcommand{\grad}{\nabla}
\newcommand{\tildeSE}{\widetilde{\mbox{SE}}}
\newcommand{\sign}{\mbox{sign}}
\newcommand{\assign}{\leftarrow}

\let\orcidlogo\relax
\usepackage{orcidlink}
\usepackage{cleveref}
\pgfplotsset{compat=1.18}
\newgeometry{margin=1in}

\begin{document}

\title[Article Title]{On solving symmetric multi-type orthogonal non-negative matrix tri-factorization problem}

\author[1]{\fnm{Rok} \sur{Hribar}\orcidlink{0000-0003-3894-5459}}\email{rok.hribar@ijs.si}%
\author[1]{\fnm{Gregor} \sur{Papa}\orcidlink{0000-0002-0623-0865}}\email{gregor.papa@ijs.si}%
\author[2,3]{\fnm{Janez} \sur{Povh}\orcidlink{0000-0002-9856-1476}}\email{janez.povh@rudolfovo.eu}%
\author*[4]{\fnm{Andrej} \sur{Kastrin}\orcidlink{0000-0002-3495-0165}}\email{andrej.kastrin@mf.uni-lj.si}%

\affil[1]{\orgdiv{Computer Systems Department}, \orgname{Jožef Stefan Institute}, \orgaddress{\street{Jamova cesta 39}, \city{Ljubljana}, \postcode{1000}, \country{Slovenia}}}%
\affil[2]{\orgdiv{Laboratory for Engineering Design}, \orgname{Faculty of Mechanical Engineering, University of Ljubljana}, \orgaddress{\street{Aškerčeva cesta 6}, \city{Ljubljana}, \postcode{1000}, \country{Slovenia}}}%
\affil[3]{\orgname{Rudolfovo -- Science and Technology Centre Novo mesto}, \orgaddress{\street{Podbreznik 15}, \city{Novo mesto}, \postcode{8000}, \country{Slovenia}}}%
\affil[4]{\orgdiv{Institute of Biostatistics and Medical Informatics}, \orgname{Faculty of Medicine, University of Ljubljana}, \orgaddress{\street{Vrazov trg 2}, \city{Ljubljana}, \postcode{1000}, \country{Slovenia}}}%

\abstract{We study the symmetric multi-type orthogonal non-negative matrix tri-factorization problem, where several symmetric non-negative matrices are simultaneously approximated by factors of the form $GS_iG^\top$, with a shared non-negative and orthogonal factor $G$. This model is motivated by clustering and network analysis, where non-negativity improves interpretability and orthogonality gives a natural assignment-type structure to the latent factor. Since the resulting optimization problem is highly non-convex, we develop two heuristic algorithms for computing high-quality local solutions. The first one is a fixed point method derived from the Karush--Kuhn--Tucker conditions after adding a penalty term for the orthogonality constraint. The second one is a three-stage ADAM-based method that combines non-negativity-preserving optimization, orthogonalization, and restricted ADAM refinement on the feasible set. We evaluate both methods on synthetic data, including noisy instances, and on citation network benchmarks. The synthetic experiments show that both algorithms recover factorizations close to the optimum and remain stable under noise. On real networks, the learned embeddings are competitive with or better than standard baselines such as SVD, node2vec, and classical link prediction heuristics in link prediction, node clustering, and node classification tasks.}


%%\pacs[JEL Classification]{D8, H51}

%%\pacs[MSC Classification]{35A01, 65L10, 65L12, 65L20, 65L70}

\maketitle

\section{Introduction}

\subsection{Motivation}

During the last decade, we have witnessed the rise of methods to reduce data dimensionality and complexity. Among them, non-negative matrix factorization (NMF) and non-negative matrix tri-factorization \eqref{eqn:NMTF} have received extraordinary attention because of their ability to
extract sparse and easily interpretable factors and  predict new associations overlooked by other methods.

The main idea of \ref{eqn:NMTF} is to factorize a given input data matrix $R$ into the product of three non-negative matrices, $G_1$, $S$, and $G_2$, with much lower dimensions. These are used for co-clustering of the concepts underlying the data and to generate the predictions (recommendations, hypotheses) for new associations between the concepts. The mathematical core of the basic variant of this method can be formulated as an optimization problem
\begin{align}
 \label{eqn:NMTF} \tag{NMTF}
  \begin{split}
  \min~  \{\|R-G_1SG_2^T\|^2 \colon G_1,G_2,S\ge 0\}
  \end{split}
\end{align}

The non-negativity constraint in \ref{eqn:NMTF} is crucial. It makes the problem harder, but it also drastically improves the usability and interpretability of the results, compared to other factorization methods, like spectral or singular value decomposition (SVD). This is further improved if we add orthogonality constraints for the outer factors $G_1$ and $G_2$, since these matrices are models for the assignment matrices, which means that the $i$-th row of $G_1$ tells us into which group of data the $i$-th data instance should be assigned. This information is encoded as the non-zero value (usually the value 1) in the column that represents the group into which the data point is assigned. Therefore, such assignment matrix $G_1$ should satisfy the constraint $G_1^TG_1=I$ and the same should also hold for $G_2$.

When the input data matrix $R$ is symmetric, which happens if we cluster the network data into homogeneous groups (communities), then we may assume that $G_1=G_2$, so we consider the symmetric non-negative matrix factorization. 
\begin{align}
 \label{eqn:SNMTF} \tag{SNMTF}
  \begin{split}
  \min~  \{\|R-GSG^T\|^2 \colon G,S\ge 0\}
  \end{split}
\end{align}

If we consider different data matrices $R_1,\ldots,R_N$, for $N>1$, above the same data set, like having different biological networks above the same set of genes, and try to cluster the underlying dataset into groups such that the best overall fit to the matrices $R_i$ is achieved, then we naturally come to symmetric multi-type orthogonal non-negative matrix tri-factorization (SONMTF) problem, introduced in the following subsection.

\subsection{Problem formulation}

The central problem that we study in this paper is \emph{symmetric multi-type orthogonal non-negative matrix tri-factorization (SONMTF)} problem. The input for this problem is an $N$-tuple of symmetric non-negative matrices $(R_1,\ldots,R_N)$, which we want to factorize simultaneously by solving:

\begin{equation}
 \label{eqn:SONMTF} 
\tag{SONMTF}
\min\Big\{ \sum_{i=1}^N \|R_{i}-GS_{i}G^\top\|^2 : G\ge 0,~S_{i}\ge 0,~G^\top G=I,~S_i=S_i^\top\Big\}.
\end{equation}
The norm in \ref{eqn:SONMTF} is the Frobenius norm.
This is a non-convex optimization problem, where the objective function is non-convex polynomial of order 6 and the feasible set is highly non-convex due to orthogonality constraint. Therefore, we cannot expect to have efficient methods to solve this problem to optimality. But this is actually the situation for all well-known clustering (community detection) models -- they are computationally intractable, hence in practice we need to use different heuristic methods to find good feasible solutions. For practical purposes, related to data science, this is usually acceptable. 
Therefore, the main goal of this paper is development of new algorithms to solve \ref{eqn:SONMTF} approximately (to find local optimum), and demonstration of their performance on synthetic and real data. 

\subsection{Related work}

The majority of the results about non-negative matrix factorization are related to the basic two-factor formulation (NMF). The recently published book by \citet{gillis2020nonnegative} contains a great overview of up-to-date theoretical and algorithmic advances pertained to this variant of the problem. The non-negative matrix tri-factorization problem \eqref{eqn:NMTF}, which is formulated above, is theoretically less understood, but its applications in machine learning are significant and well justified. It was used to fuse experimental network biological data to uncover new pan-cancer genes and to identify the most rewired genes in cancer \citep{GMDP:16a,icell2019}. New algorithms and codes for several simplified variants of NMF and \eqref{eqn:NMTF} were presented by \citet{icell2019}, \citet{asadi2020block}, and \citet{hribar2020algorithms}.

\subsubsection{Orthogonal NMTF for clustering}

A systematic analysis of tri-factor NMF was provided by \citet{DingLiPengPark2006}, with a focus on the orthogonality constraint since it leads to rigorous clustering interpretation. The authors also provided some new rules for updating the factors and proved the convergence of the algorithms. In addition, they introduced a new approach on the quality evaluation of clustering on words by using class aggregate distribution and multi-peak distribution. Multiplicative updates for orthogonal \eqref{eqn:NMTF} were proposed by \citet{Yoo2010}, where the decomposition was pursued with orthogonality constraints on the Stiefel manifold. Experiments demonstrated efficiency for document clustering and usefulness in revealing polysemous words via co-clustering words and documents. Three new orthogonal \eqref{eqn:NMTF} methods, with different error distributions---normal, Poisson, and compound Poisson---were studied by \citet{Abe2018}. They developed a $k$-means based algorithm without a multiplicative updating algorithm, and showed that the assumption of the error distribution leads to robust estimation against extremely large positive values. A correntropy-based orthogonal \eqref{eqn:NMTF} algorithm (CNMTF) was developed by \cite{peng2020robust}, which proved robust to noisy data contaminated by non-Gaussian noise and outliers, outperforming state-of-the-art methods in clustering real-world image and text datasets. The authors further elaborated this, proposing a robust \eqref{eqn:NMTF} algorithm that adopts correntropy as the similarity measure and considers double graph regularization and orthogonality conditions. A robust \eqref{eqn:NMTF} with dual hyper-graph regularization (RDHNMTF) was recently proposed by \cite{yu2025robust}, using the $\ell_{2,1}$-norm to enhance robustness to noise and outliers and introducing a dual hyper-graph to capture higher-order geometric information in both sample and feature spaces. Extensive experiments on image and text clustering benchmarks confirmed the superiority of RDHNMTF over competing methods.

\subsubsection{Symmetric and multi-type NMTF}

A symmetric \eqref{eqn:NMTF} framework for simultaneously clustering multi-type relational data was presented by \citet{WangHuangDing2011}. The proposed approach employs \eqref{eqn:NMTF} to simultaneously cluster different types of data through their inter-type relationships, and incorporates the intra-type information through manifold regularization. Similarly, \citet{Liu2015RobustMC} proposed an $L_1$-norm symmetric \eqref{eqn:NMTF} framework to cluster multi-type relational data by utilizing their inter-relatedness. Because of the $L_1$-norm distances, the approach is robust against the noise and outliers that occur in multi-relational data. In a subsequent study, \citet{LiuWang2018} proposed an $L_1$-norm symmetric \eqref{eqn:NMTF} method to solve the high-order co-clustering problem. The authors derived the algorithm using the alternating direction method of multipliers, to overcome the problem of the conventional auxiliary function approach, which does not work under orthogonal constraints and the symmetric $L_1$-norm formulation. A block inertial Bregman proximal algorithm (for minimizing the sum of a block relatively smooth function) was applied to the SNMTF by \citet{AhookhoshEtAl2021}. They also proposed some kernel functions for the SNMTF and provided closed-form solutions. A relaxation of \eqref{eqn:SONMTF} (orthogonality of $G$ and non-negativity of $S_i$ were omitted) was used to construct a computational prototype called iCell \citep{icell2019}, which integrated three omics, tissue-specific molecular interaction network types. Data about four cancers and the corresponding tissue controls was used to identify the most rewired genes and to uncover pan-cancer genes. \citet{hribar2020algorithms} developed four different algorithms to solve the symmetric multi-type \eqref{eqn:NMTF} problem, which is essentially \eqref{eqn:SONMTF} with the orthogonality constraint $G^\top G=I$ omitted.

\subsubsection{Symmetric NMF for community detection}

Highly accurate community detectors, based on graph-regularized symmetric NMF, were developed by \cite{Luo_et_al_2021}. This approach gave a significant accuracy gain over the state-of-the-art. Solving the data-mining clustering problems by the symmetric non-negative matrix factorization of the similarity matrix was studied by \cite{FavatiEtAl2019}. The authors presented variants with and without a prescribed number $k$ of clusters, and proposed a heuristic approach to improve the multi-start strategy. A joint orthogonal symmetric NMF model was proposed by \cite{kong2024joint} for community detection in complex attribute networks, unifying attribute homogeneity and topology similarity in a single framework. An orthogonality constraint is imposed on the factor matrix to improve the accuracy of node-to-community assignments, and a novel multi-order graph regularization together with an eigenvector-centrality-based enhancement strategy is established to capture comprehensive adjacency information. Most recently, a robust low-rank tensor constrained orthogonal symmetric \eqref{eqn:NMTF} method (RTOSNMF) was introduced by \cite{zhou2025robust} for multi-layer network community detection. The method separates noise from raw adjacency matrices using an $\ell_{2,1}$-norm penalty and exploits a nuclear-norm constraint to preserve the low-rank inter-layer structure of the adjacency tensor, yielding state-of-the-art performance on eight benchmark datasets.

\subsubsection{Semi-supervised and document co-clustering}

\citet{HuifangMa2010} proposed a novel semi-supervised document co-clustering model OSS-NMF via orthogonal \eqref{eqn:NMTF}. The approach incorporates prior knowledge on the document side and word side, considering dual constraints between data points (e.g., documents) and features (e.g., words), to aid the construction of new word-category and document-cluster matrices. The reported experiments show remarkable performance improvements for document clustering. An application of \eqref{eqn:NMTF} for text data co-clustering by simultaneously partitioning documents and words was done by \cite{SalahAilemNadif2018}. The authors proposed a model that maps frequently co-occurring words to approximately the same direction in the latent space to reflect the relationships among them, and derived a scalable alternating optimization algorithm with guaranteed convergence. \citet{Chen2022} introduced a parallel and scalable \eqref{eqn:NMTF}-based algorithm for text data co-clustering. They proposed solving the Lagrange dual objective function in parallel through an efficient distributed implementation. The approach is validated through five benchmark corpora for effectiveness, efficiency, and scalability.

\subsubsection{Scalable algorithms and alternative optimization strategies}

The state-of-the-art approach to solving these problems relies on the Fixed Point Method (FPM) with different variants of the multiplicative update rules \citep{cichocki2009nonnegative,vcopar2017scalable,vcopar2019fast,icell2019,peng2020robust}. Apart from FPM, some other methods from the area of mathematical optimization have been applied to \eqref{eqn:NMTF}, such as the Projected Gradient Method \citep{lin2007projected}, the Least-Squares Algorithm with the active-set method \citep{kim2008nonnegative}, and the Alternating Direction Method of Multipliers \citep[ADMM,][]{hajinezhad2016nonnegative}. A blockwise approach for latent factor learning in matrix tri-factorization was developed by \cite{vcopar2017scalable}, demonstrating almost linear scaling on multi-processor and multi-GPU architectures on large biomedical datasets. Extensive computational studies show that, alongside FPM, the Adaptive Moment Estimation Method (ADAM) -- initially developed for training deep neural networks -- is also a very efficient method for \eqref{eqn:NMTF}, especially when orthogonality constraints are imposed on factors $G_1, G_2$ and high-performance libraries such as TensorFlow are used \citep{asadi2020block,hribar2020algorithms}.

\subsubsection{Biomedical applications}

Matrix factorization has also been successfully applied in statistical modeling, particularly in the domain of recommendation systems and collaborative filtering. The most relevant for our work is the paper by \citet{lever2018collaborative}, who introduced collaborative filtering to biomedical knowledge discovery and demonstrated that the SVD algorithm substantially outperforms traditional literature-based discovery approaches. A social matrix tri-factorization model with two auxiliary matrices was developed by \citet{zhang2016nonnegative} to significantly improve the accuracy of prediction in micro-blogging; these models also capture user interests more precisely and give better prediction interpretability. A non-negative matrix tri-factorization approach for multi-omics representation learning with applications to drug repurposing and drug selection was proposed by \cite{messa2024nmtf}, exploiting a multipartite graph structure and outperforming traditional NMTF-based predictors in predicting drug--disease associations. Most recently, an improved drug repositioning model (IDDNMTF) integrating multiple heterogeneous datasets for greater precision in predicting drug--disease associations was presented in \cite{li2025drug}.

It is important to note that a majority of analytical methods for complex data rely on path-based similarity measures and thus do not fully represent the latent features of different object types and different relation types. The present work addresses this gap by tackling the full \eqref{eqn:SONMTF} problem, which simultaneously imposes non-negativity and orthogonality on a shared factor matrix across multiple symmetric relation matrices.

\subsection{Our contribution}

From the computational point of view, the basic NMF, the \ref{eqn:NMTF} and \eqref{eqn:SNMTF}, and the generalization \ref{eqn:SONMTF} are all instances of hard (NP-hard)  optimization problems \cite{vavasis2010complexity}. Moreover, due to the large size of datasets, these problems contain a huge number of variables, typically in the order of millions. Therefore, we cannot expect (unless P=NP)  to solve such problems to optimality, but we can strive to obtain solutions that are as good as possible (good local optima) using state-of-the-art techniques from non-linear optimization, and best available computing machines, i.e. supercomputers. 

In this paper, we propose two algorithms to solve \ref{eqn:SONMTF}. They are based on two well-known algorithms from non-linear optimization: on the fixed point algorithm and the adaptive moment estimation algorithm (ADAM). 
We present details about the main steps of both algorithms in \Cref{sec:fpm_adam}.
For ADAM, we developed a three-stage algorithm to meet the non-negativity of all factors and the orthogonality of $G$.
We have implemented the methods in Matlab and Python and present in \Cref{sec:num_artificial_data} numerical results, obtained on artificial (synthetic) datasets, which show that both methods recover factorizations that are close to global optimum and are stable under noise. \Cref{sec:real-data-exp} contains numerical results for three real citation network datasets, where our algorithms were used for link prediction and clustering the network nodes. Our results outperform results obtained by the baseline approaches, like the other matrix factorization methods (SVD) and link prediction heuristics (common neighbours, Jaccard coefficient, and Adamic/Adar index).

\subsection{Notations}

We denote by $\|\cdot\|_F$ or simply by $\|\cdot\|$ the Frobenius matrix norm and by $\odot$ and $\oslash$ the Hadamard matrix product and matrix division, respectively. Sometimes, we consider the entry-wise matrix 1-norm, which is the sum of absolute values of matrix entries, which we denote by $\|\cdot\|_{1,1}$. In this paper, we use two measures for factorization quality of feasible solutions for \ref{eqn:SONMTF}: the \emph{square error} ($\SE$) and \emph{mean square error} ($\MSE$), as follows:
\begin{align}
    \SE = & \sum_i\|R_i-GS_iG^\top\|^2,\\
    \MSE = & \frac{\sum_i\|R_i-GS_iG^\top\|^2}{\sum_i\|R_i\|^2}.
\end{align}
The first measure is actually the objective function of \ref{eqn:SONMTF}, while the second is a relative value of the $\SE$ compared to the size of the input data.
We use $\SE$ to define algorithms (gradients, step sizes), while $\MSE$ is used as one of the stopping criteria.

Computed solution $G,S_i$ can be infeasible for \ref{eqn:SONMTF} for two reasons: (i) some elements of $G$ or $S_i$ are negative or (ii) the columns of $G$ are not orthogonal. By construction the first situation cannot happen by the proposed methods, so we measure deviation from a feasibility for \ref{eqn:SONMTF} only by

\bea \label{def:infeas_OG}
\errG:= \frac{\| G^TG-I\|_F}{\|I\|_F}.
\eea

This measure is always non-negative and the larger it is, the less orthonormal are the columns of $G$. Unfortunately, it has no upper bound, so we cannot accurately estimate how close to feasibility the matrix $G$ is. We chose to use this measure because it is a common practice in optimization to normalise the norm of the difference between the left and right sides of constraints \citep[see e.g.,][]{mittelmann2003independent}. 

\section{Fixed point and ADAM approach to solve SONMTF}
\label{sec:fpm_adam}

\subsection{Fixed point method}

The standard method to solve several variants of non-negative matrix factorization problems approximately is to write down the Karush--Kuhn--Tucker conditions for the problem  and then reformulate them into a fixed point relation and apply iterative improvements.
In our case, we first move the orthogonality constraint $G^\top G=I$ to the objective function as the following penalty term $\frac{\alpha}{2}\|G^\top G-I\|^2$ ($\alpha$ is a positive penalty term which forces the orthogonality constraint to be approximately satisfied, the larger the $\alpha$, the more orthogonal is $G$, but the $\MSE$ may also increase as an unwanted side effect) and then develop the KKT conditions, as follows: \\

\noindent\emph{Stationarity}
\begin{align}\footnotesize
 -4\sum_i R_i G S_i + 4 \sum_i G S_i G^\top G S_i - \beta +\alpha GG^\top G-\alpha G~ =~ & 0\label{eqn:kk1} \\
 -2 G^\top R_i G +2 G^\top G S_i G^\top G - \gamma_i ~=~ & 0,~\forall i \label{eqn:kk2}
\end{align}
\emph{Primal feasibility}
\begin{align}
G\ge 0,~S_i\ge 0,~\forall i \label{eqn:kk3}
\end{align}
\emph{Dual feasibility}
\begin{align}\label{eqn:kk4} 
\beta,~\gamma_i \ge 0,~\forall i 
\end{align}
\emph{Complementary slackness}
\begin{align}
\langle \beta, G \rangle~ =~ 0 &\label{eqn:kk5}\\
\langle \gamma_i, S_i \rangle~ =~ 0 &,~\forall i \label{eqn:kk6}
\end{align}

Note that $\beta\ge 0$ and $\gamma_i\ge 0$ are dual matrix variables corresponding to $G$ and $S_i$, respectively. The Slater constraint qualification, as described by \citet{Bertsekas-2016}, is trivially satisfied for the constraints of \ref{eqn:SONMTF}, therefore it follows from the Lagrangian theory that the  candidates for a local minimum of the \ref{eqn:SONMTF} must be stationary points, which means that they must satisfy beside \eqref{eqn:kk1}--\eqref{eqn:kk2} also  the  conditions \eqref{eqn:kk3}--\eqref{eqn:kk6}.
Note that due to non-negativity of $G,S_i,\beta, \gamma_i$, we can rewrite complementarity slackness constraints as
\begin{align*}
 \beta\odot G ~=~~ \beta\odot G\odot G~=~ 0 &\\
 \gamma_i \odot S_i ~ =~ \gamma_i \odot S_i \odot S_i ~=~ 0 &,~\forall i. 
\end{align*}

By substitution $\beta, \gamma_i$ in these reformulations using \eqref{eqn:kk1}--\eqref{eqn:kk2}, we obtain the following set of equations for $G$:

\begin{align}
G\odot G \odot ( 4 \sum_i G S_i G^\top G S_i  +\alpha GG^\top G) ~=~ G\odot G \odot ( 4\sum_i R_i G S_i + \alpha G)\label{eq:compl1}\\
 S_i \odot S_i \odot ( G^\top R_i G) ~=~ 
 S_i \odot S_i \odot(G^\top G S_i G^\top G ),~\forall i.\label{eq:compl2}
\end{align}

This finally yields the following update rules:

\begin{align}
G=G\odot\sqrt{(4\sum_i R_i G S_i + \alpha G)\oslash (4 \sum_i G S_i G^\top G S_i  +\alpha GG^\top G)}\label{eq:mu1}\\
S_i=S_i\odot\sqrt{(G^\top R_i G)\oslash (G^\top G S_i G^\top G)},~\forall i \label{eq:mu2}
\end{align}
where $\sqrt{\phantom{-}}$ is applied component-wise. The fixed point method to solve \ref{eqn:SONMTF} is presented in \Cref{fig:alg_fpm}.

\begin{algorithm}
\caption{Fixed point method to solve \ref{eqn:SONMTF}}
\label{fig:alg_fpm}
\begin{algorithmic}[1]
\Require Non-negative symmetric matrices $R_1,\ldots,R_N$
\Ensure $G,\, S_1,\ldots,S_N$
\State \textbf{Initialisation:} Compute initial non-negative matrices $G$ and $S_i$ with $G^\top G=I$
\While{termination test not satisfied}
    \State Compute $G_{\text{new}}$ from current $G$ and $S_i$ using \eqref{eq:mu1}
    \State $G \leftarrow G_{\text{new}}$
    \For{$i = 1, 2, \ldots, N$}
        \State Compute $(S_i)_{\text{new}}$ from $G$ and current $S_i$ using \eqref{eq:mu2}
        \State $S_i \leftarrow (S_i)_{\text{new}}$
    \EndFor
\EndWhile
\end{algorithmic}
\end{algorithm}

\subsection{Adaptive moment estimation}

In this section we present an algorithm for solving \ref{eqn:SONMTF} that is based on Adaptive moment estimation \citep[ADAM,][]{adam}. ADAM is a gradient-based optimization algorithm widely applied for training deep neural networks. Therefore, we have empirical evidence that ADAM works well in high-dimensional search spaces and can efficiently follow (even sparse) gradients using momentum and per-parameter learning rate adaptation. However, ADAM cannot be straightforwardly applied to optimization problems with constraints. Using ADAM for solving \ref{eqn:SONMTF}, where both non-negativity and orthogonality must be ensured, requires additional techniques for guiding/restricting the algorithm to the feasible region. 
We present a novel three-stage algorithm which uses ADAM whose final solutions fully respect the imposed constraints of non-negativity and orthogonality.

From the perspective of algorithm development the two types of imposed constraints are quite different in nature. Non-negativity can be imposed by search space transformation, as shown by \cite{hribar2020algorithms}. Using this method ADAM operates in an unconstrained space while the objective function evaluation is done in a transformed space which only includes non-negative matrices. This technique works well because non-negative matrices form a single connected region. Using search space transformation it is, therefore, possible for ADAM to traverse through the non-negative part of the search space in the same manner as in the unconstrained case. If we add orthogonality constraint, however, the feasible subset breaks apart to a large number of weakly connected components. This disconnectedness of the feasible search space is a major obstruction when using local, gradient descent type algorithm such as ADAM. This is why the algorithm proposed in this section addresses the two constraints separately, by two distinct approaches. 

The proposed algorithm based on ADAM consists of three stages. In the first stage ADAM with non-negativity constraint explained in section \ref{sec:adam_nn} is executed. In this stage a high quality solution is found that respects non-negativity but not orthogonality.\footnote{We observed in the experiments, however, that the solutions found in the first stage are typically quite close to being orthogonal.} In the second stage the solution is orthogonalized so that it exactly fulfils both constraints.
In the third stage restricted ADAM with non-negativity constraint explained in section \ref{sec:restricted_adam} is executed. In this stage the orthogonal non-negative solution acquired in the second stage is refined by a type of ADAM that is restricted to the specific submanifold of orthogonal non-negative matrices found to be promising in stage two. The three stages of the algorithm are depicted in \Cref{fig:three_stages}.

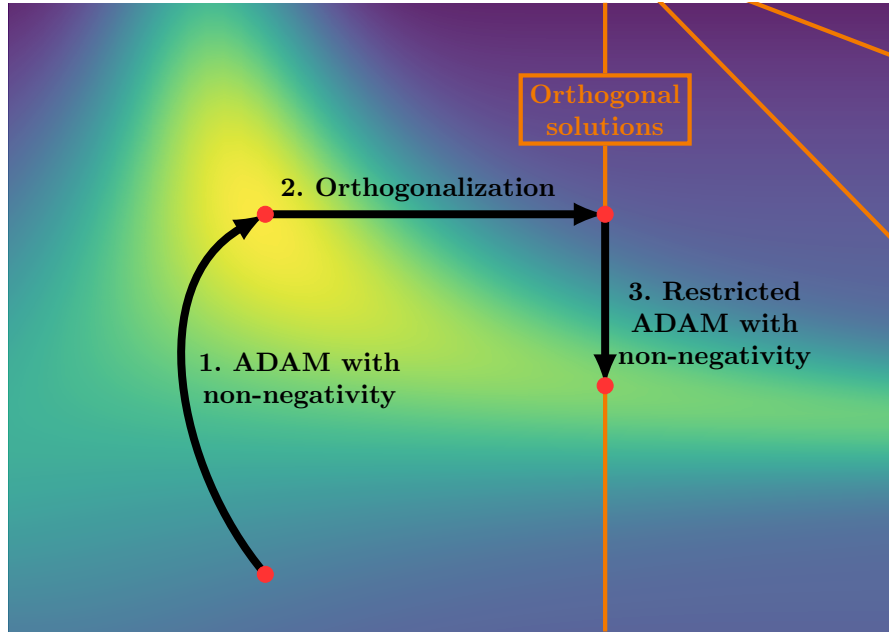
\begin{figure}[ht]
    \centering
    \newcommand{\dminx}{-50}
\newcommand{\dmaxx}{245}
\newcommand{\dminy}{-0.3}
\newcommand{\dmaxy}{0.8}
\newcommand{\gmx}{36}
\newcommand{\gmy}{0.43}
\newcommand{\fgmx}{150}
\newcommand{\fgmy}{0.13}
\begin{tikzpicture}
\begin{axis}[
             width=0.8\textwidth,
             height=0.6\textwidth,
             view={0}{90},
             hide axis,
            ]
\addplot3[opacity=20,
          surf,
          shader=interp,
          samples=50,
          colormap name=viridis,
          domain=\dminx:\dmaxx,
          domain y=\dminy:\dmaxy,
         ]{exp(-3*(y-0.5)^2-(x*sin(2*y)-0.5)^2) - 0.5*y^2};
\path[fill=white, fill opacity=0.15] (axis description cs:0,0) rectangle (axis description cs:1,1);
\node [draw=orange!95!black, text=orange!95!black, anchor=south, ultra thick, align=center] (nom) at (axis cs:\fgmx,0.85\dminy) {%
\bf Orthogonal\\ \bf solutions};
\draw [orange!95!black, ultra thick] (axis cs:\fgmx,\dminy) -- (nom);
\draw [orange!95!black, ultra thick] (nom) -- (axis cs:\fgmx,\dmaxy);
\draw [orange!95!black, ultra thick] (axis cs:300, 0.1) -- (axis cs:150,0.9);
\draw [orange!95!black, ultra thick] (axis cs:300, 0.6) -- (axis cs:150,0.9);
\draw [black, line width=3, -latex] (axis cs:\gmx, -0.2) to[out=130,in=210] node [right, align=center] {\bf 1. ADAM with \\ \bf non-negativity} (axis cs:\gmx, \gmy);
\draw [black, line width=3, -latex] (axis cs:\gmx, \gmy) -- node [above, pos=0.45] {\bf 2. Orthogonalization} (axis cs:\fgmx, \gmy);
\draw [black, line width=3, -latex] (axis cs:\fgmx, \gmy) -- node [right, pos=0.65, align=center] {\bf 3. Restricted \\ \bf ADAM with \\ \bf non-negativity} (axis cs:\fgmx,\fgmy);
\addplot[red!80, only marks, mark size=3] coordinates {(\gmx, \gmy) (\fgmx, \gmy) (\fgmx,\fgmy) (\gmx, -0.2)};
\end{axis}
\end{tikzpicture}
    \caption{Depiction of the three stages of ADAM based algorithm for finding orthogonal non-negative solutions of \ref{eqn:SONMTF}. Different positions in the plane represent different non-negative solutions of \ref{eqn:SONMTF} with hue indicating objective function (MSE) value for that solution. The orange lines show the feasible set of \ref{eqn:SONMTF} (non-negative orthogonal solutions) that geometrically consists of many weakly connected components. The black arrows show the effect that each of the three stages have in the search space.}
    \label{fig:three_stages}
\end{figure}

\subsubsection{ADAM with non-negativity constraint}
\label{sec:adam_nn}

In the first stage of the proposed algorithm ADAM is applied in a way so that the non-negativity is fulfilled throughout the descent. This setting was analyzed in depth in our previous work \citep{hribar2020algorithms} which includes further details that are omitted in this section.
To ensure non-negativity of candidate solutions when applying ADAM we used search space transformation method. In this method we introduce new search variables $\tilde{G}$ and $\tilde{S_i}$ on which no constraint is imposed and ADAM uses those variables when performing gradient descent. However, instead of optimizing objective function SE we introduce a transformed version $\tildeSE$ so that the unconstrained problem ADAM is solving is equivalent to the constrained one. This transformation is described in \Cref{tab:trans}. If ADAM finds a good solution $\tilde{G}$, $\tilde{S_i}$ of the transformed problem, we can construct an equally good solution of the original problem by applying absolute value to the variables: $G=|\tilde{G}|$ and $S_i=|\tilde{S_i}|$.

\begin{table}[!h]
\centering
\caption{Transformation of the optimization problem with non-negativity constraint to a one without constraints. Absolute value is applied element-wise.}
\label{tab:trans}
\begin{tabular*}{\textwidth}{@{\extracolsep{\fill}}lccl@{}}
\toprule
& Search variables & Constraints & \multicolumn{1}{c}{Objective function} \\
\midrule
Original problem 
& $G$, $S_i$ 
& $G\geq 0$, $S_i\geq 0$ 
& $\displaystyle\SE = \sum_i\|R_i-GS_iG^\top\|^2$ \\
Transformed problem 
& $\tilde{G}$, $\tilde{S_i}$ 
& none 
& $\displaystyle\tildeSE = \sum_i \|R_i - |\tilde{G}| |\tilde{S}_i| |\tilde{G}|^\top\|^2$ \\
\bottomrule
\end{tabular*}
\end{table}

In order to minimize $\tildeSE$ by using ADAM a gradient of the objective function is needed. The gradient can be computed using the chain rule. By introducing 
\begin{equation}\label{eq:Di}
    \D_i = R_i - |\tilde{G}| |\tilde{S}_i| |\tilde{G}|^\top,
\end{equation}
the gradient of $\tildeSE$ can be expressed as \citep{hribar2020algorithms} 
\begin{align}
    \label{math:grad1}
   \grad_{\tilde{G}}\tildeSE &= -4 \sum_i \sign(\tilde{G})\odot\left(\D_i\, |\tilde{G}| \, |\tilde{S}_i|^\top \right) \\
    \label{math:grad2}
    \grad_{\tilde{S}_i}\tildeSE &= -2 \sum_i \sign(\tilde{S}_i)\odot\left(|\tilde{G}|^\top  \D_i\, |\tilde{G}|\right),
\end{align}
where $\sign$ is the component-wise signum function. It must be stressed that $\tildeSE$ is not differentiable for all points in its domain. For such points we used sub-derivatives \citep{subgradient} with a specific choice $\mathrm{d}|x|/\mathrm{d}x=0$ when $x=0$. This is a common approach and a default mode of operation when using TensorFlow library for gradient calculation with automatic differentiation.

Using an unconstrained formulation of the problem and having a way to calculate the gradient we can apply generic ADAM algorithm to it. Below we provide the update scheme for one step of ADAM algorithm \citep{adam} when minimizing $\tildeSE$ with respect to variable $x\in\{\tilde{G}, \tilde{S_i}\}$:
\begin{align}
    \label{math:adam_mom}
    M_x &\assign \beta_1 M_x+(1-\beta_1)\, \nabla_x\tildeSE \\
    \label{math:adam_v}
    V_x &\assign \beta_2 V_x + (1 - \beta_2)\, \nabla_x\tildeSE \odot \nabla_x\tildeSE \\
    \label{math:adam_nu}
    \eta &= \alpha\frac{\sqrt{1-\beta_2^t}}{1-\beta_1^t}\\
    \label{math:adam_update}
    x &\assign x-\eta\,M_x  \oslash (\sqrt{V_x}+\varepsilon),
\end{align}
where $\oslash$ denotes element-wise division of two matrices and $\sqrt{\phantom{a}}$ and adding $\varepsilon$ are applied element-wise. Above $\alpha$, $\beta_{1,2}$ and $\varepsilon$ are parameters of the algorithm, $M_x$ and $V_x$ are exponential moving averages of the first and second moment of the gradient, respectively, and $t$ is the number of steps made. We see that ADAM is a variation of gradient descent algorithm with several additional features. Steps are not made in direction of gradient but in direction of moving average of the gradient (Eq.~\eqref{math:adam_mom}). ADAM includes per-parameter learning rate adaptation \citep[as RMSProp,][]{rmsprop} to scale step size in accordance to the average magnitudes of the gradient's components (Eq.~\eqref{math:adam_v}). Unlike RMSProp, ADAM eliminates initialization bias in first and second moment estimation (Eq.~\eqref{math:adam_nu}).

The first stage of the algorithm is detailed in \Cref{alg:three_stage_adam} (lines~\ref{alg2:line:stage1b}--\ref{alg2:line:stage1e}).
When this method is applied to solve \eqref{eqn:SONMTF}, ADAM should converge toward a solution with small objective function value SE and respecting the non-negativity constraint of $G$ and $S_i$. However, the resulting $G$ is not orthogonal. To arrive to an orthogonal solution next two stages need to be applied.

\subsubsection{Orthogonalization}

The second stage of the proposed three stage algorithm is orthogonalization in which non-negative solution is mapped to the closest possible orthogonal non-negative solution (\Cref{alg:three_stage_adam}, lines~\ref{alg2:line:stage2b}--\ref{alg2:line:stage2e}). 
This step can result in an increase of the objective function SE (usually only slightly) because orthogonalization moves the solution onto the feasible set without taking the optimization landscape into account.
The naive way to perform such mapping is to set all but the maximal component of every row of $G$ to zero. This is the minimal alteration of $G$ (with respect to any $p$-norm) that produces orthogonal non-negative $G$.
However, it is possible to align this procedure more with the nature of the underlying problem, which is to find $G$ and $S_i$ such that $GS_iG^\top\approx R_i$.
We should find a way to orthogonalize $G$ that changes $\| GS _iG^\top\|$ the least, as opposed to changing $\|G\|$ the least.
To substantiate this difference further, consider the fact that non-negative solutions $G$, $S_i$ can be freely scaled with any positive vector $u$ in the following way
\begin{align}
\label{math:Gorth}
    G &\to G\,\mbox{diag}(u) \\
\label{math:Sorth}
    S_i &\to \mbox{diag}(u)^{-1}S_i\,\mbox{diag}(u)^{-1}.
\end{align}
Such transformation of the solution $G$, $S_i$  changes neither the objective function SE
nor the predicted matrices $GS_iG^\top$. This means that the result of orthogonalization based only on the components of $G$ is dependent on how we chose to scale $G$.

To construct an unbiased procedure for orthogonalization of $G$ we need to take values of $S_i$ into account as well. 
Given the nature of the underlying factorization problem the more instinctive measure of a distance in solution space should be $\sum_i\lVert GS_iG^\top\rVert$. 
When $G$ is orthogonalized the difference between orthogonal and non-orthogonal version should be minimal with respect to this distance. 
If, instead of Frobenius norm, we take the 1-norm, the closest orthogonal non-negative $G$ can be calculated analytically. Let us calculate how much each component of $G$ contributes to this norm.
\begin{equation}
\label{math:ortho1}
    \sum_i\big\lVert\,GS_iG^\top\big\rVert_{1,1} = \sum_i\sum_{jk}\Big|\sum_{lr}G_{jl}S_{ilr}G_{kr}\Big| = \sum_{jl} G_{jl} \sum_{ikr} S_{ilr}G_{kr} = \sum_{jl} G_{jl} u_l = \big\lVert G\,\mbox{diag}(u)\big\rVert_{1,1}
\end{equation}
\begin{equation}
\label{math:ortho2}
    u_l = \sum_{ikr} S_{ilr}G_{kr}
\end{equation}
In above calculation we used the fact that $G$ and $S_i$ are non-negative and that $S_i$ are symmetric. 
We can see that  $\sum_i\lVert GS_iG^\top\rVert_{1,1}$ is in fact equivalent to $\|G\|_{1,1}$  if the solution is properly scaled using \eqref{math:Gorth}, \eqref{math:Sorth} and \eqref{math:ortho2}.
By employing scaling \eqref{math:ortho2}, the orthogonalization of $G$ can be performed in a naive fashion, by setting all but the maximal component of every row of $G$ to zero.
It is important to note that here we treat $\|\cdot\|_{1,1}$ as an approximation to Frobenius norm which, if used, does not lead to closed form solution analogous to \eqref{math:ortho1} and \eqref{math:ortho2}.

Orthogonalization is performed on a candidate solution from stage one which should have small objective value $\SE$.
After orthogonalization we get a candidate solution that now fulfills all constraints (not just non-negativity) and is as close as possible to the one with low $\SE$ value.
This procedure allowed us to identify a connected subset of the feasible set (a single orange line in Fig.~\ref{fig:three_stages}) that is close to a region of the search space with small objective value $\SE$.
In the third stage we focus our search only to this subset and due to its connectedness we can apply restricted version of ADAM algorithm in a similar fashion as in stage one.

\subsubsection{Restricted ADAM with non-negativity constraint}
\label{sec:restricted_adam}

In the last stage of the proposed three stage algorithm (\Cref{alg:three_stage_adam}, lines~\ref{alg2:line:stage3b}--\ref{alg2:line:stage3e}) we apply ADAM from the current orthogonal non-negative solution in a way that all constraints are fulfilled along the descent.
The candidate solution after orthogonalization stage has a $G$ matrix where in each row only a single component is non-zero.
If we restrict the descent so that the candidate solution cannot leave the feasible set, then only the magnitudes of the non-zero components are allowed to change.
By this we limit our descent to a connected subset of feasible solutions as was discussed in the previous section.
The positions of non-zero components of $G$ determine the feasible subset on which the descent will take place.
Technically, different feasible subsets are connected in a topological sense, however, to move from one subset to another the descent needs to pass a point where $G$ matrix includes a row with all zero values.
Such a point has necessarily high SE value which means that gradient descent is very unlikely to follow such a path.
This is why we referred to the feasible subsets as weakly connected.
Since in this stage our focus is to find a good solution in the current feasible subset such obstructions of connectedness are irrelevant, however, understanding them is crucial for justifying why three stage algorithm is needed for this problem as opposed to a single stage one.

In this stage we apply the ADAM algorithm in the same manner as in stage one but with one additional step. We process the calculated gradient so that it cannot point outside of the feasible subset. By this we ensure that the descent is restricted to move only in the feasible subset (hence the name Restricted ADAM with non-negativity constraint). To achieve this we prepare a matrix $A$ of the same size as $\tilde{G}$ that will encode which components of $\tilde{G}$ are allowed to change. Let $\tilde{G}$ be orthogonal and non-negative result of stage two, then
\begin{equation}
A_{ij} = \begin{cases} 
      1 & \tilde{G}_{ij}> 0 \\
      0 & \text{otherwise} 
   \end{cases}.
    \label{math:mask}
\end{equation}
In general the $\nabla_{\tilde{G}}\tildeSE$ can point in a direction outside of the feasible subset. To project $\nabla_{\tilde{G}}\tildeSE$ to the feasible subset $\mathcal F$ we mask the gradient in the following way:
\begin{equation}
    \mbox{proj}_{\mathcal F}\nabla_{\tilde{G}}\tildeSE = A\odot\nabla_{\tilde{G}}\tildeSE.
\end{equation}
The $\nabla_{\tilde{S_i}}\tildeSE$ does not require any projection operation since the orthogonality is imposed on $\tilde{G}$ alone. Using the projected gradient together with Adam with non-negativity should refine the solution and lower the objective function value that was possibly corrupted by the orthogonalization stage.

\begin{algorithm}[p]
\caption{Three stage algorithm based on ADAM for solving \ref{eqn:SONMTF}}
\label{alg:three_stage_adam}
\begin{algorithmic}[1]
\Require Non-negative symmetric matrices $R_1,\ldots,R_N$
\Ensure $|\tilde{G}|,\, |\tilde{S}_1|,\, \ldots,\, |\tilde{S}_N|$
\State \textbf{Initialisation:} Compute initial non-negative matrices $\tilde{G}$ and $\tilde{S_i}$ with $\tilde{G}^\top \tilde{G}=I$
\Statex
\Statex \textbf{ADAM with non-negativity constraint}
\While{termination test not satisfied} \label{alg2:line:stage1b}
    \State Compute the gradient $\nabla_{\tilde{G}}\tildeSE$ and $\nabla_{\tilde{S_i}}\tildeSE$ using \eqref{math:grad1} and \eqref{math:grad2}
    \State Update $\tilde{G}$ using \eqref{math:adam_mom}--\eqref{math:adam_update}
    \For{$i = 1, 2, \ldots, N$}
        \State Update $\tilde{S}_i$ using \eqref{math:adam_mom}--\eqref{math:adam_update}
    \EndFor
\EndWhile
\State Make the solution non-negative: $\tilde{G} \leftarrow |\tilde{G}|$,\, $\tilde{S}_1 \leftarrow |\tilde{S}_1|$,\, $\ldots$,\, $\tilde{S}_N \leftarrow |\tilde{S}_N|$ \label{alg2:line:stage1e}
\Statex
\Statex \textbf{Orthogonalization}
\State Compute scaling vector $u$ using \eqref{math:ortho2} \label{alg2:line:stage2b}
\State Scale $\tilde{G}$ and $\tilde{S}_i$ using \eqref{math:Gorth} and \eqref{math:Sorth}
\State Set all but the maximum component of every row of $\tilde{G}$ to zero \label{alg2:line:stage2e}
\Statex
\Statex \textbf{Restricted ADAM with non-negativity constraint}
\State Compute the mask $A$ using \eqref{math:mask} \label{alg2:line:stage3b}
\While{termination test not satisfied}
    \State Compute the gradients $\nabla_{\tilde{G}}\tildeSE$ and $\nabla_{\tilde{S_i}}\tildeSE$ using \eqref{math:grad1} and \eqref{math:grad2}.
    \State Project the gradient to the feasible subspace: $\nabla_{\tilde{G}}\tildeSE\assign A\odot \nabla_{\tilde{G}}\tildeSE$
    \State Update $\tilde{G}$ using \eqref{math:adam_mom}--\eqref{math:adam_update}
    \For{$i = 1, 2, \ldots, N$}
        \State Update $\tilde{S_i}$ using \eqref{math:adam_mom}--\eqref{math:adam_update}
    \EndFor
\EndWhile \label{alg2:line:stage3e}
\end{algorithmic}
\end{algorithm}

\clearpage

\section{Numerical experiments on synthetic data}
\label{sec:num_artificial_data}

\subsection{Synthetic data generation}

We tested FPM and ADAM algorithms on two data sets.
The first data set consists of synthetic random matrices that were used by \cite{hribar2020algorithms}.
This data set consists of 5-tuples of non-negative symmetric matrices $(R_1,\ldots,R_5)$ of order $n=100, 200, 500, 1000, 2000$ and $5000$. For each $n$ we first
selected $K\in \{10,20,30,40,50\}$ and constructed random non-negative $G$ of order $n\times K$ with orthogonal columns and five symmetric random non-negative matrices $S_i$ of order $K\times K$.
More precisely, for each pair $(n,K)$
we generated matrix $G$ of order $n\times K$ in two phases: (i) for each row, we randomly selected the column where the only non-zero element in this row will be set. This non-zero element is 1; (ii) in the second phase, we normalise the columns of $G$. In the first phase, we used the Matlab function {\tt randi}. The $K\times K$ matrices $S_i$ were generated by Matlab function {\tt sprandsym} and density $0.65$.
Finally, we computed $R_i=GS_iG^\top,~i=1,\ldots,5$, for each pair $(n,K)$.
These 5-tuples of matrices were input for \eqref{eqn:SONMTF}. By construction we know that if we allow the inner dimension $k\ge K$, the optimum value $\MSE_{opt}=0$.

The second data set is obtained from the first by adding to each $R_i$ a symmetric random matrix $E^\xi_i$ such that the expected values of $\|R_i+E^\xi_i-GS_iG^\top \|_F^2/\sum_i {\|R_i\|_F^2}$ are below the predefined 
$ \xi$.
For each $i$ we created $E^\xi_i$ in two steps. Firstly, we created a non-symmetric matrix $ E^\xi_i$ such that its entries were independent, uniformly distributed random numbers on the interval $[0,\tau]$ (we used 
Matlab function {\tt rand} with parameter $\tau$).
In the second step, we symmetrize it by
$E^\xi_i=(E^\xi_i+ (E^\xi_i)^\top)/2$.

\begin{lemma}
    Let  $\xi>0$ and $$\tau(\xi):=\frac{1}{n}\sqrt{\frac{2\xi }{3}\sum_i \|R_i\|_F^2}.$$
If $R_i$,   $E^\xi_i$, $G$, and $S_i$ are created as described above and  $R^\xi_i:=R_i+E^\xi_i$, then  the expected value of 
$$\sum_i\frac{\|R^\xi_i-GS_iG^\top\|_F^2}{\sum_i \|R_i\|_F^2}$$ is below $\xi$, for all $i$.
    \end{lemma}

\begin{proof}
Suppose  $\tilde E=(E+E^\top)/2$  is a random matrix, constructed as explained above, and $E=(e_{ij})_{ij}$.
Then $\mathbb{E}(\|\tilde E\|^2_F)=\sum_{i}\mathbb{E}(e_{ii}^2)+\sum_{i\neq j}\mathbb{E}((\frac{e_{ij}+e_{ji}}{2})^2)$.
By using basic properties of uniform distribution, we have $\mathbb{E}(e_{ij})=\tau/2$ and $\mathbb{E}(e_{ij}^2)=\tau^2/3$, for all $i,j$, hence
$\mathbb{E}(\|\tilde E\|^2_F)=n\tau^2/3+n(n-1)\frac{7\tau^2}{24}=(7n^2+n)\tau^2/24$. Therefore, 

\begin{align*}
\mathbb{E}\Big[\sum_i \frac{\|R^\xi_i-GS_iG^\top\|_F^2}{\sum_i \|R_i\|_F^2}\Big] = &  \mathbb{E}\Big[\sum_i \frac{\|E^\xi_i\|_F^2}{\sum_i \|R_i\|_F^2}\Big] =  \frac{5(7n^2+n)\tau^2}{24\sum_i \|R_i\|_F^2}\\ = & \frac{35n^2+5n}{24\sum_i\|R_i\|_F^2}\cdot \frac{2\xi\sum_i\|R_i\|_F^2}{3n^2}\le \xi, ~~~\mbox{for all}~n\ge 5.
\end{align*}
\end{proof}
The second data set, therefore, consists of 5-tuples of noisy matrices $R^\xi_i:=R_i+E^\xi_i$, for $\xi=10^{-6},10^{-4},10^{-2}.$

\subsection{Experimental results}

\Cref{tab:FPM_synthetic_data} contains the results obtained by FPM and ADAM on the first synthetic dataset. The computations were performed on the high-performance computing cluster at the Faculty of Mechanical Engineering, University of Ljubljana, which consists of 40 compute nodes: 20 Haswell nodes, each equipped with two Intel Xeon E5-2680 v3 12-core processors running at a base clock of 2.5 GHz and 64 GB of DDR4-2133 RAM, and 20 Rome nodes, each equipped with two AMD EPYC 7402 24-core processors running at a base clock of 2.8 GHz and 128 GB of DDR4-3200 RAM. Parallelization was achieved using the Matlab Parallel Computing Toolbox and the Python library TensorFlow.

For each $n\in\{100$, $200$, $500$, $1000$, $2000$, $5000\}$ and $K\in\{10$, $20$, $30$, $40$, $50\}$ we ran both methods with deterministic starting solution $G$, obtained from spectral decomposition of matrix $R=\sum_i R_i$ by taking the eigenvectors and eigenvalues corresponding to the $k$ largest absolute values of eigenvalues of $R$ \citep[for details see][Section 5.1]{hribar2020algorithms}. We solved \eqref{eqn:SONMTF} for each $K$ for each inner dimension $k=0.2K$, $0.4K$, $0.6K$, $0.8K$, $1.0K$, $1.2K$. By choosing various $k$ we simulate real-world situation where the true inner dimension $k$ is usually not known. Nevertheless, by construction of our synthetic data-set, we know that for $k=1.0K,1.2K$ the optimum value of $\SE$ and $\MSE$ is 0  and there exist solutions $G, S_1,\ldots,S_5$ with $\errG=0$. Our results demonstrate the capability of both methods to approach a global optimum.

\Cref{tab:FPM_synthetic_data} shows the mean values of the $\MSE$ and $\errG$, for each $n$, for each chosen value of $k_{rel}=k/K\cdot 100$, for FPM and ADAM.
The first row of this table therefore shows the mean values of $\MSE$ for $n=100$ and for $k/K=20~\%$. This means that we computed and reported for both algorithms the mean values of five $\MSE$s and $\errG$s: for pairs $k=2,K=10;~k=4,K=20;~k=6,K=30;~k=8,K=40$ and $k=10,K=50$.

We can see from \Cref{tab:FPM_synthetic_data} that FPM and ADAM are very competitive regarding the $\MSE$ and $\errG$. 
For $k_{rel}<100\%$, where the prescribed inner dimension is smaller than the true one, the $\MSE$ decreases steadily as $k_{rel}$ increases. In this regime ADAM usually gives slightly smaller values of $\MSE$, while FPM gives noticeably smaller values of $\errG$. Thus, the two methods make a different trade-off: ADAM fits the data marginally better, whereas FPM keeps the matrix $G$ closer to orthogonal.

For $k_{rel}=100\%$ both algorithms find solutions which are very close to the optimum. ADAM reaches almost zero $\MSE$ and very small $\errG$, while FPM also reaches almost zero $\MSE$, but with larger $\errG$. For $k_{rel}=120\%$, both algorithms again obtain almost zero $\MSE$. In this case the factorization is less unique because more latent components are allowed than were used to generate the data, and therefore $\errG$ is less stable. Overall, based on the $\MSE$ and $\errG$, neither method can be declared uniformly better: ADAM is slightly preferable with respect to $\MSE$, while FPM is usually preferable with respect to $\errG$.

\afterpage{%\clearpage
\begin{table}[p]
\centering
\caption{Accumulated results of FPM and ADAM on synthetic data-set. Results for FPM correspond to $\alpha=1$.}
\begin{tabular}{rr|rr|rr}
  \toprule
  $n$ & $k_{rel}$ & MSE FPM & $\errG$ FPM & MSE ADAM & $\errG$ ADAM\\ 
  \midrule
   100 &  20 & 0.5294 & 0.3485 &  0.5138 & 1.1975 \\ 
   100 &  40 & 0.4236 & 0.2779 &  0.3858 & 1.1132 \\ 
   100 &  60 & 0.2843 & 0.1845 &  0.2685 & 0.7208 \\ 
   100 &  80 & 0.1446 & 0.1138 &  0.1409 & 0.3082 \\ 
   100 & 100 & 0.0001 & 0.4222 &  0.0000 & 0.0029 \\ 
   100 & 120 & 0.0001 & 0.4941 &  0.0000 & 0.5705 \\ \midrule
   200 &  20 & 0.5318 & 0.3552 &  0.5236 & 1.2180 \\ 
   200 &  40 & 0.4166 & 0.2814 &  0.3856 & 1.2221 \\ 
   200 &  60 & 0.2883 & 0.1938 &  0.2676 & 0.6596 \\ 
   200 &  80 & 0.1481 & 0.1126 &  0.1375 & 0.3220 \\ 
   200 & 100 & 0.0001 & 0.4133 &  0.0000 & 0.0058 \\ 
   200 & 120 & 0.0001 & 0.4203 &  0.0000 & 0.5495 \\ \midrule
   500 &  20 & 0.5325 & 0.3477 &  0.5151 & 1.1170 \\ 
   500 &  40 & 0.4215 & 0.2656 &  0.3919 & 1.0901 \\ 
   500 &  60 & 0.2865 & 0.1811 &  0.2733 & 0.6885 \\ 
   500 &  80 & 0.1456 & 0.1263 &  0.1360 & 0.3657 \\ 
   500 & 100 & 0.0001 & 0.3695 &  0.0000 & 0.0058 \\ 
   500 & 120 & 0.0012 & 0.3498 &  0.0000 & 0.6125 \\ \midrule
  1000 &  20 & 0.5303 & 0.3600 &  0.5183 & 1.3755 \\ 
  1000 &  40 & 0.4207 & 0.2566 &  0.3928 & 1.2012 \\ 
  1000 &  60 & 0.2884 & 0.1742 &  0.2694 & 0.6868 \\ 
  1000 &  80 & 0.1586 & 0.1143 &  0.1393 & 0.3059 \\ 
  1000 & 100 & 0.0001 & 0.4316 &  0.0000 & 0.0058 \\ 
  1000 & 120 & 0.0002 & 0.4171 &  0.0000 & 0.6195 \\ \midrule
  2000 &  20 & 0.5107 & 0.2484 &  0.5071 & 1.0705 \\ 
  2000 &  40 & 0.3825 & 0.2128 &  0.3770 & 1.0153 \\ 
  2000 &  60 & 0.2484 & 0.1909 &  0.2578 & 0.5763 \\ 
  2000 &  80 & 0.1281 & 0.1154 &  0.1308 & 0.3082 \\ 
  2000 & 100 & 0.0023 & 0.3200 &  0.0001 & 0.0174 \\ 
  2000 & 120 & 0.0030 & 0.2159 &  0.0000 & 0.4935 \\ \midrule
  5000 &  20 & 0.5238 & 0.3681 &  0.5114 & 1.2015 \\ 
  5000 &  40 & 0.4088 & 0.2781 &  0.3790 & 1.1935 \\ 
  5000 &  60 & 0.2804 & 0.1877 &  0.2627 & 0.7038 \\ 
  5000 &  80 & 0.1558 & 0.1121 &  0.1329 & 0.3082 \\ 
  5000 & 100 & 0.0006 & 0.3304 &  0.0001 & 0.0232 \\ 
  5000 & 120 & 0.0007 & 0.3286 &  0.0001 & 0.5460 \\
  \bottomrule
\end{tabular}
\label{tab:FPM_synthetic_data}
\end{table}
}

\Cref{fig:MSE_vs_alpha} illustrates the dependence of $\MSE$ and $\errG$ on the penalty parameter $\alpha$, using the data from Table \ref{tab:MSE_vs_alpha}. As expected, larger values of $k_{rel}$ lead to smaller reconstruction errors. The parameter $\alpha$ controls the trade-off between the two criteria: smaller values of $\alpha$ favour a lower $\MSE$, but may lead to larger violations of the orthogonality constraint, whereas larger values of $\alpha$ reduce $\errG$ at the cost of a higher $\MSE$.
Among the tested values, $\alpha=100$ gives the most balanced behaviour, since the corresponding points are closest to the origin in the $(\MSE,\errG)$ plane. We therefore use $\alpha=100$ in the remaining experiments.

\afterpage{%
\begin{figure}[p]
\centering
\includegraphics[width=.9\textwidth]{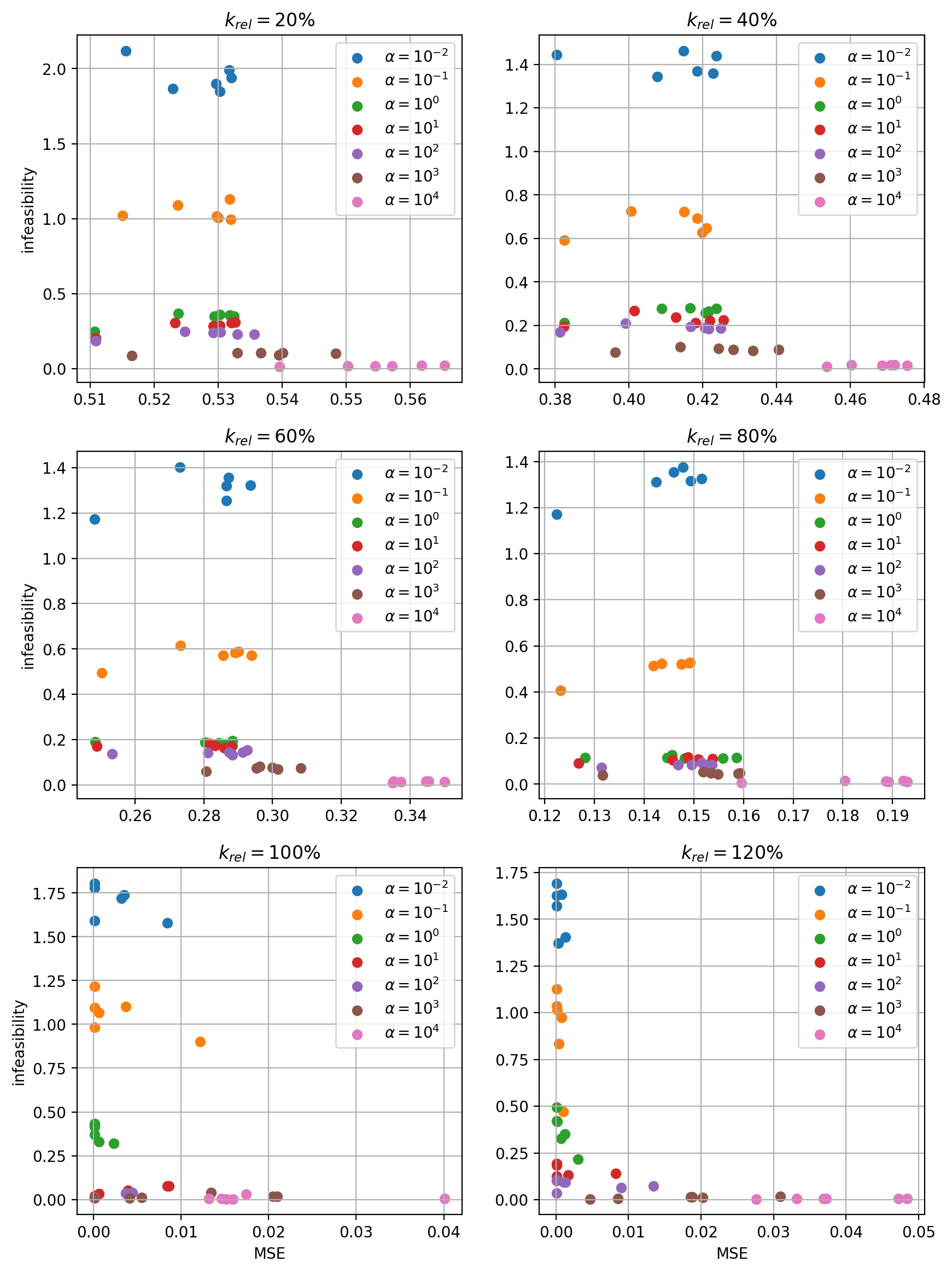}
\caption{Relation between $\MSE$ and $\errG$, computed by FPM with different $\alpha$, for all $k_{rel}=20,40,\ldots,120$.}
\label{fig:MSE_vs_alpha}
\end{figure}
}

\afterpage{%
\begin{sidewaystable}
\thisfloatpagestyle{empty}
\caption{$\MSE$ and $\errG$ for synthetic instances, for $\alpha=10^{-2},10^{-1},10^{0},10^{1},10^{2},10^{3},10^{4}$ }\label{tab:MSE_vs_alpha}
\begin{tabular}{r|r|rr|rr|rr|rr|rr|rr|rr}
\toprule
\multirow{2}{*}{$n$} & \multirow{ 2}{*}{$k_{rel}$}  & \multicolumn{2}{c|}{$\alpha=0.01$} & \multicolumn{2}{c|}{$\alpha=0.1$} & \multicolumn{2}{c|}{$\alpha=1$} & \multicolumn{2}{c|}{$\alpha=10$} & \multicolumn{2}{c|}{$\alpha=100$} & \multicolumn{2}{c|}{$\alpha=1000$} & 
\multicolumn{2}{c}{$\alpha=10000$} \\  
      &    & $\MSE$ & $\errG$& $\MSE$ & $\errG$& $\MSE$ & $\errG$& $\MSE$ & $\errG$& $\MSE$ & $\errG$& $\MSE$ & $\errG$& $\MSE$ & $\errG$ \\
\midrule
100 & 20 & 0.5303 & 1.8491 & 0.5300 & 1.0044 & 0.5294 & 0.3485 & 0.5293 & 0.2820 & 0.5293 & 0.2397 & 0.5367 & 0.1049 & 0.5545 & 0.0176 \\ 
  100 & 40 & 0.4227 & 1.3596 & 0.4198 & 0.6268 & 0.4236 & 0.2779 & 0.4219 & 0.2234 & 0.4215 & 0.1845 & 0.4336 & 0.0854 & 0.4709 & 0.0186 \\ 
  100 & 60 & 0.2866 & 1.2545 & 0.2939 & 0.5714 & 0.2843 & 0.1845 & 0.2859 & 0.1638 & 0.2927 & 0.1531 & 0.3015 & 0.0684 & 0.3373 & 0.0139 \\ 
  100 & 80 & 0.1424 & 1.3123 & 0.1419 & 0.5125 & 0.1446 & 0.1138 & 0.1509 & 0.1074 & 0.1468 & 0.0848 & 0.1534 & 0.0481 & 0.1929 & 0.0112 \\ 
  100 & 100 & 0.0001 & 1.7792 & 0.0122 & 0.9011 & 0.0001 & 0.4222 & 0.0001 & 0.0134 & 0.0037 & 0.0372 & 0.0134 & 0.0393 & 0.0174 & 0.0306 \\ 
  100 & 120 & 0.0001 & 1.5723 & 0.0001 & 1.1260 & 0.0001 & 0.4941 & 0.0001 & 0.1945 & 0.0001 & 0.1041 & 0.0188 & 0.0149 & 0.0332 & 0.0058 \\ \midrule
  200 & 20 & 0.5317 & 1.9915 & 0.5318 & 1.1314 & 0.5318 & 0.3552 & 0.5321 & 0.3036 & 0.5330 & 0.2269 & 0.5395 & 0.0891 & 0.5572 & 0.0169 \\  
  200 & 40 & 0.4148 & 1.4613 & 0.4149 & 0.7221 & 0.4166 & 0.2814 & 0.4127 & 0.2373 & 0.4167 & 0.1937 & 0.4243 & 0.0952 & 0.4685 & 0.0175 \\ 
  200 & 60 & 0.2873 & 1.3562 & 0.2892 & 0.5832 & 0.2883 & 0.1938 & 0.2832 & 0.1729 & 0.2883 & 0.1322 & 0.3000 & 0.0759 & 0.3455 & 0.0160 \\ 
  200 & 80 & 0.1515 & 1.3251 & 0.1492 & 0.5269 & 0.1481 & 0.1126 & 0.1509 & 0.1085 & 0.1536 & 0.0845 & 0.1519 & 0.0526 & 0.1922 & 0.0138 \\ 
  200 & 100 & 0.0035 & 1.7391 & 0.0001 & 1.2178 & 0.0001 & 0.4133 & 0.0001 & 0.0207 & 0.0001 & 0.0113 & 0.0001 & 0.0074 & 0.0159 & 0.0031 \\ 
  200 & 120 & 0.0001 & 1.6906 & 0.0001 & 1.0360 & 0.0001 & 0.4203 & 0.0001 & 0.1836 & 0.0013 & 0.0947 & 0.0202 & 0.0114 & 0.0484 & 0.0053 \\ \midrule
  500 & 20 & 0.5321 & 1.9398 & 0.5320 & 0.9952 & 0.5325 & 0.3477 & 0.5327 & 0.3071 & 0.5357 & 0.2294 & 0.5484 & 0.1014 & 0.5654 & 0.0193 \\ 
  500 & 40 & 0.4236 & 1.4401 & 0.4210 & 0.6469 & 0.4215 & 0.2656 & 0.4255 & 0.2244 & 0.4249 & 0.1866 & 0.4405 & 0.0898 & 0.4753 & 0.0158 \\ 
  500 & 60 & 0.2866 & 1.3197 & 0.2856 & 0.5708 & 0.2865 & 0.1811 & 0.2881 & 0.1709 & 0.2874 & 0.1447 & 0.2954 & 0.0738 & 0.3447 & 0.0163 \\ 
  500 & 80 & 0.1494 & 1.3173 & 0.1475 & 0.5217 & 0.1456 & 0.1263 & 0.1488 & 0.1162 & 0.1516 & 0.0937 & 0.1549 & 0.0433 & 0.1892 & 0.0103 \\ 
  500 & 100 & 0.0001 & 1.5910 & 0.0037 & 1.1025 & 0.0001 & 0.3695 & 0.0084 & 0.0763 & 0.0001 & 0.0107 & 0.0055 & 0.0103 & 0.0152 & 0.0029 \\ 
  500 & 120 & 0.0001 & 1.6261 & 0.0008 & 0.9753 & 0.0012 & 0.3498 & 0.0001 & 0.1260 & 0.0011 & 0.0926 & 0.0047 & 0.0043 & 0.0472 & 0.0053 \\ \midrule
  1000 & 20 & 0.5297 & 1.9010 & 0.5298 & 1.0170 & 0.5303 & 0.3600 & 0.5303 & 0.2885 & 0.5304 & 0.2438 & 0.5401 & 0.1049 & 0.5618 & 0.0193 \\ 
  1000 & 40 & 0.4185 & 1.3677 & 0.4184 & 0.6938 & 0.4207 & 0.2566 & 0.4180 & 0.2135 & 0.4205 & 0.1893 & 0.4282 & 0.0897 & 0.4720 & 0.0189 \\ 
  1000 & 60 & 0.2936 & 1.3226 & 0.2901 & 0.5876 & 0.2884 & 0.1742 & 0.2864 & 0.1607 & 0.2913 & 0.1435 & 0.3083 & 0.0738 & 0.3501 & 0.0149 \\ 
  1000 & 80 & 0.1478 & 1.3750 & 0.1490 & 0.5260 & 0.1586 & 0.1143 & 0.1538 & 0.1095 & 0.1523 & 0.0836 & 0.1589 & 0.0457 & 0.1886 & 0.0131 \\ 
  1000 & 100 & 0.0032 & 1.7192 & 0.0006 & 1.0667 & 0.0001 & 0.4316 & 0.0039 & 0.0522 & 0.0037 & 0.0371 & 0.0041 & 0.0092 & 0.0146 & 0.0048 \\ \midrule
  1000 & 120 & 0.0013 & 1.4058 & 0.0002 & 1.0141 & 0.0002 & 0.4171 & 0.0001 & 0.1026 & 0.0001 & 0.0349 & 0.0085 & 0.0068 & 0.0373 & 0.0068 \\ 
  2000 & 20 & 0.5156 & 2.1175 & 0.5151 & 1.0217 & 0.5107 & 0.2484 & 0.5109 & 0.2033 & 0.5109 & 0.1837 & 0.5165 & 0.0868 & 0.5396 & 0.0141 \\ 
  2000 & 40 & 0.3804 & 1.4450 & 0.3825 & 0.5931 & 0.3825 & 0.2128 & 0.3823 & 0.1950 & 0.3812 & 0.1697 & 0.3962 & 0.0760 & 0.4535 & 0.0119 \\ 
  2000 & 60 & 0.2482 & 1.1730 & 0.2503 & 0.4938 & 0.2484 & 0.1909 & 0.2489 & 0.1711 & 0.2534 & 0.1364 & 0.2807 & 0.0604 & 0.3350 & 0.0098 \\ 
  2000 & 80 & 0.1224 & 1.1712 & 0.1232 & 0.4072 & 0.1281 & 0.1154 & 0.1268 & 0.0915 & 0.1314 & 0.0708 & 0.1316 & 0.0378 & 0.1596 & 0.0062 \\ 
  2000 & 100 & 0.0084 & 1.5782 & 0.0001 & 0.9839 & 0.0023 & 0.3200 & 0.0086 & 0.0759 & 0.0001 & 0.0118 & 0.0210 & 0.0173 & 0.0132 & 0.0064 \\ 
  2000 & 120 & 0.0003 & 1.3717 & 0.0010 & 0.4723 & 0.0030 & 0.2159 & 0.0017 & 0.1315 & 0.0134 & 0.0733 & 0.0186 & 0.0143 & 0.0276 & 0.0042 \\ \midrule
  5000 & 20 & 0.5229 & 1.8672 & 0.5237 & 1.0900 & 0.5238 & 0.3681 & 0.5233 & 0.3043 & 0.5248 & 0.2469 & 0.5330 & 0.1058 & 0.5503 & 0.0181 \\ 
  5000 & 40 & 0.4077 & 1.3445 & 0.4006 & 0.7257 & 0.4088 & 0.2781 & 0.4014 & 0.2665 & 0.3990 & 0.2094 & 0.4139 & 0.1020 & 0.4603 & 0.0187 \\ 
  5000 & 60 & 0.2730 & 1.4020 & 0.2732 & 0.6155 & 0.2804 & 0.1877 & 0.2818 & 0.1823 & 0.2812 & 0.1412 & 0.2963 & 0.0805 & 0.3352 & 0.0152 \\ 
  5000 & 80 & 0.1459 & 1.3550 & 0.1435 & 0.5232 & 0.1558 & 0.1121 & 0.1457 & 0.1057 & 0.1496 & 0.0836 & 0.1593 & 0.0491 & 0.1804 & 0.0146 \\ 
  5000 & 100 & 0.0001 & 1.8040 & 0.0001 & 1.0947 & 0.0006 & 0.3304 & 0.0006 & 0.0319 & 0.0044 & 0.0399 & 0.0205 & 0.0158 & 0.0401 & 0.0060 \\ 
  5000 & 120 & 0.0008 & 1.6317 & 0.0004 & 0.8336 & 0.0007 & 0.3286 & 0.0082 & 0.1408 & 0.0090 & 0.0649 & 0.0309 & 0.0167 & 0.0369 & 0.0051 \\ 
\bottomrule
\end{tabular}
\end{sidewaystable}
}

Numerical results obtained by FPM and ADAM on noisy synthetic data are presented in \Cref{tab:noisy_synthetic_unconstrained} and \Cref{tab:noisy_synthetic_orthogonal}. For all computations in this experiment, we use $\alpha = 100$.
The noisy experiments show the same qualitative behaviour as the noiseless ones. As $k_{rel}$ increases, the $\MSE$ decreases, and for $k_{rel}\ge 100\%$ the residuals are very small. For $k_{rel}<100\%$, the error caused by using too small  inner dimension dominates the effect of the added noise, so the results change only mildly with respect to $\xi$. Without enforcing the orthogonality constraint, FPM and ADAM are almost indistinguishable, with ADAM being slightly better in most cases. With the orthogonality constraint, FPM gives smaller $\MSE$ for most under-parametrized cases, while ADAM remains very accurate around the true dimension $k_{rel}=100\%$. Hence, adding noise does not change the main conclusion from the noiseless experiment: both algorithms are stable, and neither one dominates the other in all regimes.

\afterpage{%
\begin{table}[p]
\centering
\caption{Solving \eqref{eqn:SONMTF} without orthogonality constraint with FPM and ADAM on noisy synthetic data ($n=500$), using $\alpha=100$.}\label{tab:noisy_synthetic_unconstrained}
\begin{tabular}{rr|rrr|rrr}
& & & FPM & & &ADAM& \\
\toprule
$k$ & $k_{rel}$ & $\xi=10^{-6}$ & $\xi=10^{-4}$ & $\xi=10^{-2}$ & $\xi=10^{-6}$ & $\xi=10^{-4}$ & $\xi=10^{-2}$ \\
\midrule
$10$ & $20 $&$ 0.5074 $&$ 0.5034 $&$ 0.4643$ & $0.5074$ & $0.5311$ & $0.4643$ \\
$10$ & $40 $&$ 0.3776 $&$ 0.3743 $&$ 0.3439$ & $0.3620$ & $0.3589$ & $0.3305$ \\
$10$ & $60 $&$ 0.2343 $&$ 0.2317 $&$ 0.2121$ & $0.2379$ & $0.2319$ & $0.2337$ \\
$10$ & $80 $&$ 0.1109 $&$ 0.0979 $&$ 0.0902$ & $0.0984$ & $0.0974$ & $0.0895$ \\
$10$& $100 $&$ 0.0001 $&$ 0.0001 $&$ 0.0019$ & $0.0000$ & $0.0000$ & $0.0019$ \\
$10$& $120 $&$ 0.0001 $&$ 0.0001 $&$ 0.0019$ & $0.0000$ & $0.0000$ & $0.0013$ \\\midrule
$20$ & $20 $&$ 0.5163 $&$ 0.5024 $&$ 0.4716$ & $0.5167$ & $0.5078$ & $0.4689$ \\
$20$ & $40 $&$ 0.3907 $&$ 0.3890 $&$ 0.3602$ & $0.3889$ & $0.3894$ & $0.3542$ \\
$20$ & $60 $&$ 0.2739 $&$ 0.2592 $&$ 0.2433$ & $0.2628$ & $0.2604$ & $0.2400$ \\
$20$ & $80 $&$ 0.1475 $&$ 0.1249 $&$ 0.1309$ & $0.1366$ & $0.1310$ & $0.1144$ \\
$20$& $100 $&$ 0.0001 $&$ 0.0001 $&$ 0.0018$ & $0.0000$ & $0.0000$ & $0.0018$ \\
$20$& $120 $&$ 0.0001 $&$ 0.0001 $&$ 0.0013$ & $0.0000$ & $0.0000$ & $0.0013$ \\\midrule
$30$ & $20 $&$ 0.5017 $&$ 0.5116 $&$ 0.4575$ & $0.5062$ & $0.4946$ & $0.4577$ \\
$30$ & $40 $&$ 0.3945 $&$ 0.3949 $&$ 0.3667$ & $0.3936$ & $0.3863$ & $0.3541$ \\
$30$ & $60 $&$ 0.2771 $&$ 0.2776 $&$ 0.2591$ & $0.2777$ & $0.2700$ & $0.2475$ \\
$30$ & $80 $&$ 0.1443 $&$ 0.1378 $&$ 0.1310$ & $0.1474$ & $0.1407$ & $0.1363$ \\
$30$& $100 $&$ 0.0001 $&$ 0.0257 $&$ 0.0017$ & $0.0000$ & $0.0000$ & $0.0017$ \\
$30$& $120 $&$ 0.0001 $&$ 0.0001 $&$ 0.0013$ & $0.0000$ & $0.0000$ & $0.0013$ \\\midrule
$40$ & $20 $&$ 0.5058 $&$ 0.5028 $&$ 0.4648$ & $0.5017$ & $0.5036$ & $0.4576$ \\
$40$ & $40 $&$ 0.3964 $&$ 0.3981 $&$ 0.3599$ & $0.3945$ & $0.3908$ & $0.3591$ \\
$40$ & $60 $&$ 0.2810 $&$ 0.2785 $&$ 0.2529$ & $0.2720$ & $0.2725$ & $0.2473$ \\
$40$ & $80 $&$ 0.1500 $&$ 0.1434 $&$ 0.1355$ & $0.1432$ & $0.1416$ & $0.1317$ \\
$40$& $100 $&$ 0.0001 $&$ 0.0001 $&$ 0.0018$ & $0.0000$ & $0.0000$ & $0.0017$ \\
$40$& $120 $&$ 0.0001 $&$ 0.0001 $&$ 0.0013$ & $0.0000$ & $0.0000$ & $0.0013$ \\\midrule
$50$ & $20 $&$ 0.5265 $&$ 0.5251 $&$ 0.4854$ & $0.5237$ & $0.5190$ & $0.4791$ \\
$50$ & $40 $&$ 0.4305 $&$ 0.4187 $&$ 0.3910$ & $0.4185$ & $0.4173$ & $0.3817$ \\
$50$ & $60 $&$ 0.3036 $&$ 0.3000 $&$ 0.2813$ & $0.3002$ & $0.3003$ & $0.2728$ \\
$50$ & $80 $&$ 0.1554 $&$ 0.1647 $&$ 0.1452$ & $0.1565$ & $0.1545$ & $0.1416$ \\
$50$& $100 $&$ 0.0029 $&$ 0.0001 $&$ 0.0017$ & $0.0000$ & $0.0000$ & $0.0017$ \\
$50$& $120 $&$ 0.0001 $&$ 0.0001 $&$ 0.0013$ & $0.0000$ & $0.0000$ & $0.0012$ \\
\bottomrule
\end{tabular}
\end{table}
}

\afterpage{%
\clearpage
\begin{table}[p]
\centering
\caption{Solving \eqref{eqn:SONMTF} FPM and ADAM on noisy synthetic data ($n=500$), using $\alpha=100$.}
\begin{tabular}{rr|rrr|rrr}
& & & FPM & & &ADAM& \\
\toprule
$k$ & $k_{rel}$ & $\xi=10^{-6}$ & $\xi=10^{-4}$ & $\xi=10^{-2}$ & $\xi=10^{-6}$ & $\xi=10^{-4}$ & $\xi=10^{-2}$ \\
\midrule
$10$ & $20$ &$0.5081 $& $0.5311$ &$ 0.4903$ & $0.5232$ & $0.5410$ & $0.4789$ \\
$10$ & $40$ &$0.3830 $& $0.3799$ &$ 0.3349$ & $0.4013$ & $0.3982$ & $0.3690$ \\
$10$ & $60$ &$0.2497 $& $0.2473$ &$ 0.2264$ & $0.2742$ & $0.2699$ & $0.2887$ \\
$10$ & $80$ &$0.1124 $& $0.1169$ &$ 0.1076$ & $0.1228$ & $0.1219$ & $0.1154$ \\
$10$ &$100$ &$0.0001 $& $0.0001$ &$ 0.0019$ & $0.0000$ & $0.0000$ & $0.0033$ \\
$10$ &$120$ &$0.0001 $& $0.0001$ &$ 0.0019$ & $0.0000$ & $0.0000$ & $0.0031$ \\\midrule
$20$ & $20$ &$0.5190 $& $0.5136$ &$ 0.4749$ & $0.5475$ & $0.5491$ & $0.4950$ \\
$20$ & $40$ &$0.4030 $& $0.3916$ &$ 0.3574$ & $0.4330$ & $0.4294$ & $0.4043$ \\
$20$ & $60$ &$0.2647 $& $0.2624$ &$ 0.2452$ & $0.3053$ & $0.3030$ & $0.2836$ \\
$20$ & $80$ &$0.1481 $& $0.1415$ &$ 0.1299$ & $0.1607$ & $0.1553$ & $0.1396$ \\
$20$ &$100$ &$0.0001 $& $0.0329$ &$ 0.0312$ & $0.0000$ & $0.0000$ & $0.0033$ \\
$20$ &$120$ &$0.0001 $& $0.0001$ &$ 0.0016$ & $0.0000$ & $0.0000$ & $0.0031$ \\\midrule
$30$ & $20$ &$0.5152 $& $0.5108$ &$ 0.4640$ & $0.5450$ & $0.5315$ & $0.5003$ \\
$30$ & $40$ &$0.4004 $& $0.3888$ &$ 0.3583$ & $0.4428$ & $0.4307$ & $0.3912$ \\
$30$ & $60$ &$0.2833 $& $0.2807$ &$ 0.2518$ & $0.3162$ & $0.3051$ & $0.2888$ \\
$30$ & $80$ &$0.1468 $& $0.1500$ &$ 0.1339$ & $0.1643$ & $0.1611$ & $0.1566$ \\
$30$ &$100$ &$0.0001 $& $0.0330$ &$ 0.0258$ & $0.0000$ & $0.0000$ & $0.0033$ \\
$30$ &$120$ &$0.0001 $& $0.0001$ &$ 0.0016$ & $0.0000$ & $0.0000$ & $0.0030$ \\\midrule
$40$ & $20$ &$0.5080 $& $0.5058$ &$ 0.4700$ & $0.5444$ & $0.5483$ & $0.5034$ \\
$40$ & $40$ &$0.4015 $& $0.3929$ &$ 0.3651$ & $0.4488$ & $0.4406$ & $0.4135$ \\
$40$ & $60$ &$0.2843 $& $0.2757$ &$ 0.2533$ & $0.3119$ & $0.3087$ & $0.2913$ \\
$40$ & $80$ &$0.1474 $& $0.1443$ &$ 0.1333$ & $0.1617$ & $0.1620$ & $0.1556$ \\
$40$ &$100$ &$0.0001 $& $0.0001$ &$ 0.0179$ & $0.0000$ & $0.0000$ & $0.0033$ \\
$40$ &$120$ &$0.0001 $& $0.0001$ &$ 0.0015$ & $0.0000$ & $0.0000$ & $0.0030$ \\\midrule
$50$ & $20$ &$0.5299 $& $0.5263$ &$ 0.4856$ & $0.5704$ & $0.5649$ & $0.5298$ \\
$50$ & $40$ &$0.4322 $& $0.4264$ &$ 0.3922$ & $0.4674$ & $0.4625$ & $0.4263$ \\
$50$ & $60$ &$0.3026 $& $0.3005$ &$ 0.2799$ & $0.3333$ & $0.3344$ & $0.3097$ \\
$50$ & $80$ &$0.1619 $& $0.1552$ &$ 0.1435$ & $0.1726$ & $0.1712$ & $0.1629$ \\
$50$ &$100$ &$0.0001 $& $0.0001$ &$ 0.0157$ & $0.0000$ & $0.0000$ & $0.0032$ \\
$50$ &$120$ &$0.0001 $& $0.0001$ &$ 0.0017$ & $0.0000$ & $0.0000$ & $0.0030$ \\
\bottomrule
\end{tabular}
\label{tab:noisy_synthetic_orthogonal}
\end{table}%
\clearpage
}

\section{Numerical experiments on real data}
\label{sec:real-data-exp}

In this section, we present the results of a series of experiments demonstrating the performance of the proposed method using real data. Specifically, the goal of the experiments is to evaluate whether the orthogonal non-negative latent representations learned by SONMTF improve downstream graph mining tasks. We evaluate the method on three common graph mining tasks: (i) link prediction, (ii) node clustering, and (iii) node classification. \Cref{subsec:real-data} introduces the data sets used in this part of the study. \Cref{subsec:exp-setup} describes the experimental design settings applied to all three evaluation tasks. Finally, \Crefrange{subsec:lpr-task}{subsec:ncla-task} detail each machine learning method and present the experimental results.

\subsection{Datasets}%
\label{subsec:real-data}

To evaluate the proposed SONMTF approach on real-world graph mining tasks, we used three widely adopted citation network benchmarks: Cora, CiteSeer, and PubMed \citep{yang2016revisiting}. These datasets differ substantially in size, density, and number of classes, providing a diverse test bed for assessing the robustness and generalizability of the learned node representations. Moreover, citation networks naturally exhibit community structure and homophily, making them particularly suitable for evaluating matrix factorization methods designed to uncover latent relational patterns.

\Cref{tab:bench_net_props} summarizes the main structural characteristics of the benchmark networks. The \emph{Cora} citation network consists of 2708 scientific publications connected by 5278 citation links. Each publication belongs to one of seven research topics, making the dataset suitable for evaluating representation learning methods in node classification and clustering tasks. The \emph{CiteSeer} dataset contains 3327 scientific publications connected through 4676 citation relationships and annotated with six topic categories. Compared with Cora, the network is sparser and therefore provides a more challenging benchmark for representation learning methods. The \emph{PubMed} citation network was extracted from a collection of scientific publications related to diabetes research. It is the largest dataset considered in this study, comprising 19{,}717 nodes and 44{,}327 citation links. Each publication is assigned to one of three classes.

\begin{table}[h]
\centering
\caption{Structural characteristics of the benchmark citation networks}
\begin{tabular}{l*2{S[table-format=5.0]}S[table-format=1.2]c}
\toprule
Dataset & {Nodes} & {Edges} & {Degree} & {Classes} \\
\midrule
Cora & 2708 & 5278 & 3.89 & 7 \\
CiteSeer & 3327 & 4676 & 2.81 & 6 \\
PubMed  & 19717 & 44327 & 4.49 & 3 \\
\bottomrule        
\end{tabular}
\label{tab:bench_net_props}
\end{table}

\subsection{Experimental setup}%
\label{subsec:exp-setup}

To apply a machine learning algorithm to a graph, we must first generate suitable features representing each node. These features are usually extracted manually based on the topology of the graph, such as degree centrality or path length. However, manually creating features is time-consuming and complicated for modern, complex datasets. Fortunately, many embedding approaches have been developed to enable automatic feature extraction from graph data. The basic idea behind an embedding algorithm is to learn a mapping function that embeds each node in a low-dimensional $\mathbb{R}^d$ space as a feature vector by optimizing an objective function to represent the structure of the input network.

From an algorithmic perspective, network embedding approaches can be divided into three groups: matrix factorization-based, random walk-based, and deep learning-based methods. For a comprehensive overview, we refer interested readers to the recent surveys \citep[e.g.,][]{zhou2022network}. Following this taxonomy, we conduct the present experiment in two steps:
\begin{enumerate}
    \item First we apply the proposed non-negative matrix factorization approach to map a given network (i.e., we solve \eqref{eqn:SONMTF} for a single input adjacency matrix $A$) to get an embedding represented by the matrix $G$ with $N$ rows and $P$ columns. We set $P = 128$, following standard practice in graph embedding benchmarks~\citep{grover2016node}. In the second step, we run a selected machine learning algorithm on a matrix $X$.
    \item We adopt two baseline approaches that are commonly used in graph embedding applications, namely SVD and node2vec algorithm.
\end{enumerate}

The well-known SVD is an example of matrix factorization methods. For SVD, we used a truncated version of the algorithm (tSVD) because it is computationally more efficient, especially for sparse matrices. Additionally, tSVD can produce a resulting matrix with $d \ll P$ columns, while standard SVD is limited to $P$ features~\citep{halko2011finding}. Node2vec, in contrast, uses a random walk to capture the structure of a network and then learns node embeddings by training a skip-gram model~\citep{grover2016node}.

\subsection{Link prediction}
\label{subsec:lpr-task}

(General) Link prediction (LP) is a task in modern data science that aims to predict whether a connection between a pair of nodes will occur \citep{liben2007link}. The intuition behind LP is that more similar nodes are more likely to be connected. For a detailed introduction to link prediction, see survey by \citet{kumar2020link}.

(Details) The evaluation approach for LP follows the procedure described in our previous work \citep{kastrin2016link}. In brief, we randomly removed 30\% of the links while keeping the network connected. This step provides the training network, to which we apply two groups of feature construction methods:
\begin{enumerate}
    \item Representation learning: We apply methods designed for matrix factorization (SNMTF, SONMTF, SVD) and random walks (node2vec) to the training network to generate $d$-dimensional latent vectors for each node. To convert node-level representations to edge-level features for the classifier, we compute the Hadamard product of the two node vectors, producing a single element-wise multiplied link vector representing the edge.
    \item Topology-based heuristics: We calculate traditional link prediction heuristic scores that measure structural similarity directly from the training graph. These include neighborhood-overlap metrics like common neighbors (CN), Jaccard coefficient (JC), and Adamic/Adar index (AA).
\end{enumerate}

The resulting link vectors and raw heuristic scores are fed into a logistic regression binary classifier. The prediction performance is evaluated on the hidden test edges alongside an equivalent set of randomly sampled non-existent edges, measured via the Area Under the Receiver Operating Characteristic (AUROC) metric.

The results in \Cref{tab:link-prediction} show that representation learning approaches outperform traditional topology-based heuristics across all evaluated citation networks. Notably, the proposed SONMTF formulations consistently achieve performance comparable to or exceeding that of both tSVD and node2vec. On the Cora and CiteSeer datasets, the non-orthogonal and orthogonal variants attain the highest predictive accuracy, with the non-orthogonal model performing best on Cora (0.789) and the orthogonal model achieving the top score on CiteSeer (0.792). On the larger PubMed dataset, the non-orthogonal variant again yields the strongest performance (0.887). These results indicate that combining non-negativity constraints with the factorization of shared structural patterns enables more accurate modeling of link relationships than unconstrained linear decompositions or random walk–based embeddings.

\begin{table}[h]
\sisetup{separate-uncertainty=true}
\centering
\caption{Link prediction results (AUROC) on the benchmark citation networks}
\begin{tabular}{l*3{S[table-format=1.3(3)]}}
\toprule
Method & {Cora} & {CiteSeer} & {PubMed} \\
\midrule
CN & 0.639 \pm 0.009 & 0.603 \pm 0.020 & 0.580 \pm 0.033 \\
JC & 0.638 \pm 0.031 & 0.603 \pm 0.041 & 0.580 \pm 0.035 \\
AA & 0.639 \pm 0.043 & 0.603 \pm 0.038 & 0.580 \pm 0.021 \\
\midrule
tSVD & 0.770 \pm 0.005 & 0.770 \pm 0.025 & 0.836 \pm 0.049 \\
node2vec & 0.687 \pm 0.021 & 0.747 \pm 0.021 & 0.852 \pm 0.010 \\
SNMTF & 0.789 \pm 0.002 & 0.790 \pm 0.043 & 0.887 \pm 0.034 \\
SONMTF & 0.782 \pm 0.037 & 0.792 \pm 0.020 & 0.868 \pm 0.006 \\
\bottomrule
\end{tabular}
\footnotetext{\textit{Legend}: CN, common neighbors; JC, Jaccard coefficient; AA, Adamic/Adar index.}
\footnotetext{\textit{Note}: Results are presented as mean AUROC\,$\pm$\,SD over 10 runs. Larger values imply better performance.}
\label{tab:link-prediction}
\end{table}

\subsection{Node clustering}
\label{subsec:ncu-task}

(General) Node clustering (NC) is an unsupervised task aimed at partitioning nodes into multiple homogeneous clusters. In NC we first learn the node representation vectors of a graph, and then perform unsupervised partitioning to get the final clustering solution. The objective is to ensure that nodes within the same cluster (community) are more densely connected to each other than to nodes in other clusters.

(Details) The evaluation pipeline for NC is executed through the following steps:
\begin{enumerate}
\item Representation learning: Similar to the LP task, we apply matrix factorization (tSVD) and random-walk-based method (node2vec) to the network to map nodes into a continuous vector space. NC operates directly on node-level embeddings.
\item Node partitioning: We apply the unsupervised $k$-means clustering algorithm to the resulting embeddings.
\item Cluster assessment: We evaluate the quality of the resulting clusters using the silhouette metric \citep{rousseeuw1987silhouettes}.
\end{enumerate}

The pipeline was run 10 times using random seed initializations. The individual silhouette widths from these trials were averaged, with the resulting mean values reported as the final performance scores in \Cref{tab:node-clustering}.

\begin{sidewaystable}[p]
\thisfloatpagestyle{empty}
\sisetup{separate-uncertainty=true}
\centering
\caption{Node clustering results (silhouette width) for different embedding methods and numbers of clusters.}
\label{tab:node-clustering}
\begin{tabular}{l*9{S[table-format=1.3(3)]}}
\toprule
& \multicolumn{9}{c}{Cora} \\
\cmidrule{2-10}
Method/$k$ &   {2} &   {3} &   {4} &   {5} &   {6} &   {7} &   {8} &   {9} & {10} \\
\cmidrule{1-10}
tSVD                   & 0.506 \pm 0.004 & 0.485 \pm 0.006 & 0.503 \pm 0.005 & 0.505 \pm 0.007 & 0.521 \pm 0.003 & 0.527 \pm 0.004 & 0.535 \pm 0.006 & 0.540 \pm 0.005 & 0.538 \pm 0.006 \\
node2vec              & 0.138 \pm 0.012 & 0.077 \pm 0.015 & 0.094 \pm 0.011 & 0.104 \pm 0.009 & 0.105 \pm 0.014 & 0.108 \pm 0.008 & 0.107 \pm 0.010 & 0.115 \pm 0.013 & 0.119 \pm 0.011 \\
SNMTF    & 0.515 \pm 0.003 & 0.541 \pm 0.004 & 0.543 \pm 0.005 & 0.564 \pm 0.003 & 0.577 \pm 0.004 & 0.572 \pm 0.006 & 0.581 \pm 0.005 & 0.593 \pm 0.002 & 0.582 \pm 0.004 \\
SONMTF       & 0.978 \pm 0.001 & 0.760 \pm 0.006 & 0.689 \pm 0.008 & 0.473 \pm 0.011 & 0.498 \pm 0.009 & 0.508 \pm 0.010 & 0.534 \pm 0.007 & 0.546 \pm 0.006 & 0.554 \pm 0.008 \\
\midrule
& \multicolumn{9}{c}{CiteSeer} \\
\cmidrule{2-10}
Method/$k$ &   {2} &   {3} &   {4} &   {5} &   {6} &   {7} &   {8} &   {9} & {10} \\
\cmidrule{1-10}
tSVD                   & 0.418 \pm 0.005 & 0.418 \pm 0.007 & 0.420 \pm 0.004 & 0.419 \pm 0.006 & 0.394 \pm 0.009 & 0.394 \pm 0.006 & 0.386 \pm 0.008 & 0.384 \pm 0.011 & 0.388 \pm 0.005 \\
node2vec              & 0.249 \pm 0.014 & 0.287 \pm 0.011 & 0.247 \pm 0.015 & 0.252 \pm 0.010 & 0.254 \pm 0.009 & 0.275 \pm 0.012 & 0.287 \pm 0.013 & 0.287 \pm 0.011 & 0.241 \pm 0.014 \\
SNMTF               & 0.418 \pm 0.003 & 0.399 \pm 0.004 & 0.399 \pm 0.006 & 0.400 \pm 0.004 & 0.386 \pm 0.005 & 0.372 \pm 0.007 & 0.395 \pm 0.004 & 0.370 \pm 0.006 & 0.361 \pm 0.008 \\
SONMTF                  & 0.423 \pm 0.004 & 0.419 \pm 0.003 & 0.415 \pm 0.005 & 0.388 \pm 0.008 & 0.411 \pm 0.006 & 0.409 \pm 0.007 & 0.381 \pm 0.007 & 0.373 \pm 0.009 & 0.375 \pm 0.006 \\
\midrule
& \multicolumn{9}{c}{PubMed} \\
\cmidrule{2-10}
Method/$k$ &   {2} &   {3} &   {4} &   {5} &   {6} &   {7} &   {8} &   {9} & {10} \\
\cmidrule{1-10}
tSVD                   & 0.443 \pm 0.006 & 0.332 \pm 0.008 & 0.332 \pm 0.007 & 0.352 \pm 0.006 & 0.330 \pm 0.008 & 0.312 \pm 0.009 & 0.337 \pm 0.005 & 0.310 \pm 0.007 & 0.299 \pm 0.010 \\
node2vec              & 0.270 \pm 0.011 & 0.279 \pm 0.009 & 0.270 \pm 0.012 & 0.107 \pm 0.015 & 0.184 \pm 0.013 & 0.104 \pm 0.014 & 0.110 \pm 0.012 & 0.147 \pm 0.011 & 0.125 \pm 0.013 \\
SNMTF               & 0.498 \pm 0.002 & 0.402 \pm 0.004 & 0.402 \pm 0.003 & 0.440 \pm 0.005 & 0.314 \pm 0.006 & 0.353 \pm 0.004 & 0.345 \pm 0.005 & 0.358 \pm 0.004 & 0.358 \pm 0.005 \\
SONMTF                  & 0.499 \pm 0.002 & 0.404 \pm 0.003 & 0.405 \pm 0.004 & 0.443 \pm 0.004 & 0.351 \pm 0.005 & 0.358 \pm 0.005 & 0.350 \pm 0.006 & 0.364 \pm 0.005 & 0.364 \pm 0.004 \\
\bottomrule
\end{tabular}
\footnotetext{\textit{Note}: Results are presented as mean silhouette width\,$\pm$\,SD over 10 runs. Larger values imply better performance.}
\end{sidewaystable}

The clustering results detailed in \Cref{tab:node-clustering} reveal distinct behavior patterns among the embedding methods. For the Cora dataset, the orthogonal SONMTF model achieves a high silhouette score at lower cluster bounds (0.978 for $k = 2$ and 0.760 for $k = 3$), confirming that imposing explicit orthogonality constraints transforms the continuous latent space into a sparse, near-discrete representation that aligns exceptionally well with low-cardinality global community structures. As the number of clusters increases, the non-orthogonal variant tends to capture finer, overlapping sub-structures more effectively, matching the behavior seen in PubMed dataset where both SONMTF methods outperform tSVD and node2vec across nearly all values of $k$. Across all datasets, node2vec yields significantly lower silhouette widths, indicating that its continuous random walk spaces lack the explicit geometric boundaries naturally induced by NMTF.

\subsection{Node classification}
\label{subsec:ncla-task}

(General) Node classification (NC) is a supervised learning task whose objective is to predict the unobserved labels (i.e., classes) of nodes based on their features and the underlying network structure. This task leverages known ground-truth labels and structural homophily, operating under the assumption that interconnected nodes, or nodes with high structural similarity, are likely to share the same class labels.

(Details) The node classification pipeline follows a standard supervised learning protocol:
\begin{enumerate}
  \item Representation learning: Node-level latent features are extracted using tSVD, node2vec, and the proposed non-orthogonal (SNMFT) and orthogonal (SONMTF) models, mapping the underlying network graph into a continuous $d$-dimensional representation space.
  \item Model training and evaluation: The network nodes are partitioned into a training set containing known class labels and a held-out test set. Logistic regression is then trained on the learned embeddings, and the AUROC is evaluated against the ground-truth labels.
\end{enumerate}

\begin{table}[htb]
\sisetup{separate-uncertainty=true}
\centering
\caption{Node classification results (AUROC) on the benchmark citation networks.}
\begin{tabular}{l*3{S[table-format=1.3(3)]}}
\toprule
Method & {Cora} & {CiteSeer} & {PubMed} \\
\midrule
tSVD & 0.572 \pm 0.029 & 0.832 \pm 0.020 & 0.894 \pm 0.032 \\
node2vec & 0.572 \pm 0.018 & 0.816 \pm 0.058 & 0.776 \pm 0.064 \\
SNMTF & 0.610 \pm 0.002 & 0.848 \pm 0.003 & 0.859 \pm 0.016 \\
SONMTF & 0.626 \pm 0.026 & 0.844 \pm 0.024 & 0.915 \pm 0.032 \\
\bottomrule
\end{tabular}
\footnotetext{\textit{Note}: Results are presented as mean AUROC\,$\pm$\,SD over 10 runs. Larger values imply better performance.}
\label{tab:node-classification}
\end{table}

The supervised classification performance summarized in \Cref{tab:node-classification} reinforces the descriptive power of the features generated by our proposed models. The orthogonal SONMTF framework achieves the highest AUROC on both the Cora (0.626) and PubMed (0.915) datasets. The PubMed result is particularly notable, as the orthogonal variant beats the non-orthogonal model by a wide margin and outpaces the strong tSVD baseline (0.894), highlighting that the discrete assignment-like constraints within the common factor $G$ act as an effective regularizer that prevents overfitting during supervised downstream classification. For CiteSeer, the non-orthogonal variant achieves the top performance (0.848), demonstrating that even without strict geometric enforcement, the tri-factorization layer successfully distills robust semantic features from relational data.

Overall, orthogonality encourages more discriminative latent representations and sharper cluster assignment. Across all tasks and datasets, SONMTF ranks among the strongest methods. This is particularly evident for node classification task, where SONMTF achieves the strongest performance on two of the three benchmark datasets.

\section{Concluding remarks}

In this paper we studied the symmetric multi-type orthogonal non-negative matrix tri-factorization problem \eqref{eqn:SONMTF}. The model is motivated by applications in which several symmetric relation matrices over the same set of objects have to be approximated simultaneously, while the common factor $G$ provides a low-dimensional and interpretable representation of the objects. The non-negativity constraint supports interpretability, whereas the orthogonality constraint gives the factor $G$ an assignment-like structure, which is particularly natural in clustering and network analysis.

We proposed two heuristic approaches for solving \eqref{eqn:SONMTF}. The first one is a fixed point method derived from the KKT conditions after adding a penalty term for the orthogonality constraint. This leads to simple multiplicative update rules and gives a direct way to balance reconstruction accuracy and orthogonality through the parameter $\alpha$. The second one is a three-stage ADAM-based method. Since ADAM cannot be applied directly to the non-negative orthogonal feasible set, the proposed algorithm first solves a non-negative transformed problem, then orthogonalizes the obtained solution, and finally refines it by restricted ADAM on the selected feasible component.

The experiments on synthetic data show that both approaches are able to recover high-quality factorizations. When the prescribed inner dimension is at least as large as the true generating dimension, both methods obtain solutions with almost zero reconstruction error. For smaller inner dimensions, the error increases as expected, but the behaviour remains stable and consistent. The experiments also illustrate the trade-off induced by the penalty parameter $\alpha$: larger values improve orthogonality, while too large values may deteriorate the reconstruction error. Among the tested values, $\alpha=100$ gave the best compromise. The noisy synthetic experiments further indicate that both algorithms are robust with respect to moderate perturbations of the input matrices.

The real-data experiments on citation networks confirm that the proposed factorization-based embeddings are useful for standard graph mining tasks. In link prediction, the proposed methods outperform the considered topology-based heuristics and are also competitive with, or better than, SVD and node2vec. In node classification, the embeddings obtained with the orthogonality constraint give the best results on Cora and PubMed, while the non-orthogonal variant is slightly better on CiteSeer. The clustering results are more mixed, but they still show that the proposed approach is competitive and that the usefulness of the orthogonality constraint depends on the data set and on the prescribed embedding dimension.

Overall, the results suggest that symmetric multi-type orthogonal non-negative matrix tri-factorization is a viable model for learning interpretable node representations from multi-relational symmetric data. The fixed point method is simple and transparent, while the ADAM-based approach provides a flexible alternative that can exploit modern gradient-based optimization tools. Neither method dominates the other in all settings, which is consistent with the non-convex nature of the problem, but both produce useful solutions in synthetic and real-data experiments.

Several directions remain open for future work. A more detailed convergence analysis of the proposed algorithms would be valuable, especially for the fixed point method with the orthogonality penalty and for the restricted ADAM refinement step. Another important issue is the automatic selection of the inner dimension $k$ and of the penalty parameter $\alpha$. From the application point of view, it would be useful to test the methods on larger and more heterogeneous multi-layer networks, and to extend the model to directed, weighted, or partially observed relation matrices. Finally, more scalable implementations could make the approach applicable to substantially larger networks.

\section*{Declarations}

\subsection*{Funding}

The authors acknowledge financial support from the Slovenian Research and Innovation Agency (ARIS) through the core research programmes Computer Structures and Systems (R.H. and G.P., Grant No. P2-0098) and Methodology for Data Analysis in Medical Sciences (A.K., Grant No. P3-0154), the research project Quantum Solver for Hard Binary Quadratic Problems (J.P., Grant No. J7-50186), and the annual work programme of Rudolfovo (J.P.).

\subsection*{Conflict of interest/Competing interests}
The authors have no financial or proprietary interests in any material discussed in this article.

\subsection*{Ethics approval and consent to participate}
Not applicable.

\subsection*{Consent for publication}
Not applicable.

\subsection*{Data availability}
The datasets used in this study are publicly available in the GitLab repository at \url{https://repo.ijs.si/hribarr/on-solving-symmetric-multi-type-orthogonal-nmtf-problem}.

\subsection*{Materials availability}
Not applicable.

\subsection*{Code availability} 
The source code used in this study is available in the GitLab repository at \url{https://repo.ijs.si/hribarr/on-solving-symmetric-multi-type-orthogonal-nmtf-problem}.

\subsection*{Author contribution}
All authors contributed to the study conception and design, material preparation, data collection, analysis, methodology development, interpretation of the results, and manuscript writing. All authors read and approved the final manuscript.

\bibliography{references}%
\end{document}